\documentclass[11pt]{article}

\usepackage{amsmath, amssymb, amsthm, graphicx, color, geometry, setspace}
\usepackage[utf8]{inputenc}
\usepackage[T1]{fontenc}
\usepackage{lmodern}
\usepackage{booktabs, url, hyperref, tabularx}
\graphicspath{{figures/}}
\usepackage{tikz}
\usetikzlibrary{arrows.meta, positioning, shapes.geometric}
\usepackage{algorithm, algpseudocode}
\usepackage{makecell}
\newtheorem{theorem}{Theorem}[section]

\newtheorem{corollary}{Corollary}[section]

\title{Multi-Scale Dequant: Eliminating Dequantization Bottleneck via Activation Decomposition for Efficient LLM Inference}
\author{
Lingchao Zheng \and Yuwei Fan\thanks{corresponding author: fanyuwei2@huawei.com} \and Jun Li \and Chengqiu Hu \and
  Qichen Liao \and Junyi Fan \and Rui Shi \and Fangzheng Miao
}
\date{\today}

\begin{document}
\maketitle

\begin{abstract}
Quantization is essential for efficient large language model (LLM) inference, yet the dequantization step---converting low-bit weights back to high-precision for matrix multiplication---has become a critical bottleneck on modern AI accelerators. On architectures with decoupled compute units (e.g., Ascend NPUs), dequantization operations can consume more cycles than the matrix multiplication itself, leaving the high-throughput tensor cores underutilized.

This paper presents \textbf{Multi-Scale Dequant (MSD)}, a quantization framework that removes weight/KV dequantization from the GEMM critical path. Instead of lifting low-bit weights to BF16 precision, MSD decomposes high-precision BF16 activations into multiple low-precision components, each of which can be multiplied directly with quantized weights via native hardware-accelerated GEMM. This approach shifts the computational paradigm from \emph{precision conversion} to \emph{multi-scale approximation}, avoiding INT8-to-BF16 weight conversion before GEMM.

We instantiate MSD for two weight formats and derive tight error bounds for each. For INT8 weights (W8A16), two-pass INT8 decomposition achieves $\sim$16 effective bits with error bound $M/64516 \approx M/2^{16}$. For MXFP4 weights (W4A16), two-pass MXFP4 decomposition yields $\sim$6.6 effective bits with error bound $\alpha/64$ per block---surpassing single-pass MXFP8's 5.24 bits while maintaining the same effective GEMM compute time. We further derive closed-form latency and HBM traffic models showing that MSD avoids the Vector-Cube pipeline stall caused by dequantization and reduces KV cache HBM traffic by up to $2.5\times$ in attention. Numerical simulations on matrix multiplication and Flash Attention kernels confirm that MSD does not degrade accuracy compared to dequantization baselines, and in many settings achieves lower L2 error.
\end{abstract}
\section{Introduction}
\label{sec:intro}

Large language models (LLMs) have demonstrated remarkable capabilities across a wide range of tasks, but their deployment at scale presents significant computational challenges. The inference cost of state-of-the-art models is dominated by memory bandwidth and matrix multiplication operations, particularly during the decode phase where activations are computed sequentially. Quantization---representing weights and/or activations with fewer bits---is the predominant technique for reducing memory footprint and accelerating inference.

\subsection{The Dequantization Bottleneck}

The dominant quantization paradigm for LLM inference employs low-bit weights (INT8 or INT4) combined with high-precision activations (BF16 or FP16). This asymmetric design requires \emph{dequantization}: converting quantized weights back to high precision before matrix multiplication. On modern AI accelerators, this step has emerged as a critical performance bottleneck.

A concrete illustration comes from DeepSeek's FlashMLA FP8 kernel on NVIDIA Hopper GPUs~\cite{flashmla-fp8-blog}. Profiling reveals that the dequantization path---converting \texttt{float8\_e4m3} $\to$ \texttt{half} $\to$ \texttt{float32} $\to$ \texttt{bfloat16} followed by scale multiplication---consumes approximately 50 clock cycles per KV token, while the matrix multiply-accumulate (MMA) operations for 64 query heads require only 34 cycles. The kernel is therefore \emph{dequantization-bound}: Tensor Cores sit idle while CUDA Cores struggle to feed them with dequantized data.

This phenomenon is even more pronounced on architectures with decoupled compute units, such as Huawei Ascend NPUs. The Ascend 910B architecture features a heterogeneous compute structure comprising Vector cores (for scalar/vector operations) and Cube cores (for high-throughput matrix operations). Dequantization---which involves element-wise type conversion and scaling---must execute on Vector cores, while the subsequent GEMM utilizes Cube cores. The disparity in throughput between these units creates a severe pipeline stall: Cube cores wait for Vector cores to complete dequantization, resulting in significant underutilization of the accelerator's peak compute capacity. A recent study on W4A16 kernels for Ascend 910~\cite{w4a16-ascend} independently confirms that ``the primary bottleneck is not dequantization computation itself, but extra global memory transfer for the weight''---i.e., the round-trip HBM traffic caused by dequantization dominates latency. Similar dequantization-bound behavior has been observed on NVIDIA GPUs~\cite{flashmla-fp8-blog, qserve}.

The dequantization bottleneck has attracted considerable attention. On NVIDIA GPUs, QServe~\cite{qserve} reports 20--90\% runtime overhead from INT4 dequantization and proposes compute-aware weight reordering to mitigate it; LiquidGEMM~\cite{liquidgemm} achieves up to $2.9\times$ speedup over prior W4A8 kernels by deferring dequantization to the GEMM epilogue; TurboMind~\cite{turbomind} further optimizes the mixed-precision pipeline through offline weight packing and fused dequantization. These efforts, while effective, share a common limitation: they \emph{optimize} the dequantization path rather than \emph{eliminating} it.

\subsection{The MSD Approach}

We propose \textbf{Multi-Scale Dequant (MSD)}, a fundamental rethinking of the quantization workflow. Rather than quantizing weights to low precision and then dequantizing them during inference, MSD preserves weights in INT8 format and instead \emph{decomposes BF16 activations into multiple INT8 components}.

Specifically, for an activation vector $x \in \mathbb{R}^n$ in BF16 format and a weight matrix $W \in \mathbb{Z}_8^{m \times n}$ in INT8 format, MSD computes:
\begin{equation}
x \approx \alpha \cdot x^{(1)} + \beta \cdot x^{(2)}, \quad x^{(i)} \in \mathbb{Z}_8^n,
\end{equation}
where $\alpha, \beta \in \mathbb{R}$ are scaling coefficients. The matrix multiplication is then performed as:
\begin{equation}
W x \approx \alpha \cdot (W x^{(1)}) + \beta \cdot (W x^{(2)}).
\end{equation}

Critically, both $W$ and $x^{(i)}$ are in INT8 format, allowing $W x^{(i)}$ to execute as native INT8$\times$INT8 GEMM on hardware tensor cores (e.g., Ascend Cube cores, NVIDIA Tensor Cores) without any dequantization step. The final result is reconstructed by scaling and summing the two partial outputs.

This approach is conceptually related to---yet fundamentally different from---ABQ-LLM~\cite{abq-llm}, which decomposes \emph{quantized weights} into binary components via Binary TensorCore (BTC) equivalents. Both methods replace a single mixed-precision GEMM with multiple uniform-precision GEMMs followed by scaled accumulation; however, MSD decomposes the \emph{activation} side, which is more natural on architectures where activations arrive in high precision and weights are already in low-precision format.

\subsection{Contributions}

Our contributions are threefold:
\begin{itemize}
    \item \textbf{Dequantization-free quantization framework:} To our knowledge, MSD is the first activation-side multi-scale decomposition framework targeting dequantization bottlenecks on decoupled low-precision inference architectures. By decomposing activations rather than lifting weights, MSD removes weight/KV dequantization from the GEMM critical path.
    \item \textbf{Theoretical analysis across precision formats:} We derive tight error bounds for MSD across weight formats: $\sim$16 effective bits from two INT8 passes (W8A16, error bound $M/64516$), and $\sim$6.6 effective bits from two MXFP4 passes (W4A16, error bound $\alpha/64$ per block)---surpassing single-pass MXFP8's 5.24 bits. We also derive closed-form latency models showing that MSD avoids the Vector-Cube dequantization bottleneck while maintaining comparable effective Cube compute time.
    \item \textbf{Numerical validation:} We conduct element-level accuracy simulations demonstrating that MSD does not degrade accuracy compared to dequantization baselines for both GEMM and Flash Attention kernels, and in many settings achieves lower L2 error---for both INT8 (W8A16) and MXFP4 (W4A16) weight formats.
\end{itemize}

\section{Background}
\label{sec:background}

\subsection{Huawei Ascend NPU Architecture}
\label{sec:ascend}

The Huawei Ascend 910B NPU is a massively parallel AI accelerator designed for training and inference workloads. Understanding its architectural characteristics is essential for appreciating the dequantization bottleneck and the design of MSD.

\subsubsection{Heterogeneous Compute Units}

The Ascend 910B features two distinct types of compute units:

\textbf{Vector Cores.} The Vector unit executes scalar and vector operations including element-wise arithmetic, type conversions, and memory access patterns. It operates at lower throughput compared to the Cube unit and is typically used for preprocessing, activation functions, and data movement.

\textbf{Cube Cores.} The Cube unit is Ascend's tensor accelerator, providing high-throughput matrix multiplication via systolic arrays. On Ascend 910B, the Cube unit delivers up to 256 TFLOPS of FP16/BF16 throughput and $2\times$ higher INT8 throughput. Critically, the Cube unit supports native INT8$\times$INT8 GEMM with accumulation to INT32, enabling high-efficiency quantized computation.

\subsubsection{Memory Hierarchy}

The Ascend memory hierarchy consists of:
\begin{itemize}
    \item \textbf{HBM (High Bandwidth Memory):} 32--64 GB capacity, ~1 TB/s bandwidth
    \item \textbf{L0 Buffer:} Software-managed scratchpad for tiling
    \item \textbf{Unified Buffer (UB):} On-chip buffer for Vector core data; not directly accessible by Cube cores
\end{itemize}

Data movement between HBM and compute units is orchestrated via DMA engines, with double-buffering used to overlap transfer and computation.

\subsubsection{The Dequantization Problem on Ascend}

The heterogeneous architecture creates a fundamental mismatch for dequantization-heavy workloads:
\begin{enumerate}
    \item INT8 weights must be loaded from HBM
    \item Vector cores perform type conversion (INT8 $\to$ BF16/FP16) and scale multiplication
    \item Converted weights are written back to HBM (or UB)
    \item Cube cores read from HBM/UB and perform GEMM
\end{enumerate}

\textbf{The CV Communication Bottleneck.} On Ascend 910B, Vector cores and Cube cores have separate on-chip caches (L0 Buffer / Unified Buffer) that are not directly shared. Data produced by Vector cores must be written to HBM before Cube cores can access it, and vice versa. Each dequantization operation requires:
\begin{itemize}
    \item \textbf{Read:} Load INT8 weights from HBM to Vector registers
    \item \textbf{Compute:} Vector cores perform INT8 $\to$ BF16 type conversion + scaling
    \item \textbf{Write:} Write converted BF16 weights back to HBM/UB
    \item \textbf{Read again:} Cube cores read BF16 weights for GEMM
\end{itemize}

This round-trip HBM communication \textbf{doubles the memory bandwidth consumption} compared to a single read. For a weight matrix of size $m \times n$, dequantization alone reads $mn$ bytes (INT8) and writes $2mn$ bytes (BF16), creating severe memory pressure that limits achievable utilization.

\subsection{The Microscaling (MX) Specification}
\label{sec:mx-spec}

The Microscaling (MX) data format~\cite{microscaling} is an emerging standard for low-precision machine learning computation. MX defines a block-based quantization scheme where each block of $B = 32$ elements shares a single scale factor stored in E8M0 format (8-bit exponent, zero mantissa, i.e., a power of two). Element values use fixed-point or floating-point formats such as INT8, FP8 (E4M3/E5M2), or FP4 (E1M2).

The E8M0 shared scale constraint---requiring $\alpha$ and $\beta$ to be powers of two---has important implications for MSD design. Unlike INT8 MSD where $\alpha = M/127$ can take any value, MXFP4 MSD must select $\alpha = 2^{\lceil \log_2(M_b / c) \rceil}$ for some constant $c$, restricting the scaling granularity. This constraint motivates the $\alpha$-relaxation optimization described in Section~\ref{sec:optimization}.

The FP4 E1M2 format used in MXFP4 has representable positive values $\{0, 0.25, 0.5, 0.75, 1.0, 1.25, 1.5, 1.75\}$ with a \emph{uniform step size of 0.25}. This uniformity simplifies the error analysis: the maximum rounding error for any in-range value is exactly half the step size (0.125), as established in Theorem~\ref{thm:error-mxfp4}.

\subsection{Related Work}
\label{sec:related}

\subsubsection{Weight-Only Quantization}

GPTQ~\cite{gptq} and AWQ~\cite{awq} achieve 3--4 bit weight compression with minimal accuracy loss. However, these methods require dequantizing weights to FP16/BF16 before GEMM, incurring the overhead described above. Recent work on Marlin~\cite{marlin} and similar kernels optimizes the dequantization-GEMM fusion on GPUs but does not eliminate the fundamental bottleneck. On Ascend NPUs, He et al.~\cite{w4a16-ascend} present the first practical W4A16 kernel using Vector cores for on-the-fly dequantization and Split-K parallelization, yet report that the redundant HBM transfer remains the dominant cost.

\subsubsection{Activation Quantization}

SmoothQuant~\cite{smoothquant} migrates quantization difficulty from activations to weights via per-channel smoothing, enabling W8A8 INT8 GEMM without dequantization. However, SmoothQuant requires calibration and may suffer accuracy degradation on certain models. LLM.int8()~\cite{llm-int8} handles activation outliers through mixed-precision decomposition but still relies on dequantization for the non-outlier components.

\subsubsection{Dequantization-Aware Kernel Optimization}

A growing body of work targets the dequantization bottleneck through kernel-level optimization. QServe~\cite{qserve} introduces W4A8KV4 quantization with compute-aware weight reordering and register-level parallelism to reduce dequantization latency on GPUs. LiquidGEMM~\cite{liquidgemm} redesigns the W4A8 GEMM pipeline to defer dequantization to the epilogue phase, achieving up to $2.9\times$ speedup. TurboMind~\cite{turbomind} provides a comprehensive mixed-precision inference framework with offline weight packing and fused dequantization. MixPE~\cite{mixpe} proposes performing dequantization \emph{after} per-group integer GEMM, reducing the overhead through shift-and-add operations rather than multipliers. These approaches optimize the dequantization path but do not eliminate it.

\subsubsection{Alternative Computation Paradigms}

Several works seek to bypass dequantization entirely through alternative computational strategies. ABQ-LLM~\cite{abq-llm} decomposes quantized weights into binary components and reconstructs arbitrary-precision GEMM via Binary TensorCore (BTC) equivalents, achieving acceleration for non-standard bit-widths such as W6A6 and W2A8. LUT-GEMM~\cite{park2024} and LUT Tensor Core~\cite{lut-tc} replace dequantization with lookup-table-based computation, precomputing partial dot products to avoid explicit type conversion. T-MAN~\cite{t-man} extends the LUT approach to NPUs with a unified table layout for both prefill and decoding. FIGNA~\cite{figna} takes a hardware design approach, proposing dedicated FP-INT multiply-accumulate units that natively support mixed-precision operations without dequantization.

DQT~\cite{dqt} introduces a nested integer representation where lower-precision values are bit-wise embedded within higher-precision ones, enabling dequantization-free precision switching via bit-shift operations in the \emph{training} context.

\subsubsection{Hardware-Aware Kernel Design}

AMLA~\cite{amla} introduces optimized FlashAttention kernels for Ascend NPUs, achieving high FLOPS utilization through hierarchical tiling and pipelining. However, AMLA focuses on attention computation and does not address the quantized GEMM bottleneck in linear layers. FlashMLA~\cite{deepseek-flashmla} optimizes memory-efficient attention with FP8 KVCache but, as discussed, remains dequantization-bound.

\subsubsection{Positioning of MSD}

MSD differs from all the above approaches in a fundamental way. While weight-only quantization methods (GPTQ, AWQ) require dequantization, kernel optimization methods (QServe, LiquidGEMM, TurboMind) reduce its cost, and alternative paradigms (ABQ-LLM, LUT-based methods) circumvent it through different computational primitives, MSD \emph{removes weight/KV dequantization from the GEMM critical path through a tightly bounded activation decomposition}. The key insight is to keep weights in their native low-precision format (INT8 or MXFP4) and instead decompose the high-precision activation into multiple low-precision components, enabling pure low-precision GEMM on standard hardware---no custom arithmetic units, no lookup tables, no binary decomposition of weights. Among existing methods, ABQ-LLM is the closest in spirit: both replace mixed-precision GEMM with multiple uniform-precision GEMMs. However, ABQ-LLM decomposes on the \emph{weight} side (bit-level), whereas MSD decomposes on the \emph{activation} side (value-level), which is better suited for architectures where weights are pre-quantized and activations are computed on-the-fly.

\section{Method}
\label{sec:method}

This section presents the Multi-Scale Dequant (MSD) framework. We first formalize the decomposition problem, then describe the optimization strategy for computing scaling coefficients, and finally detail the hardware mapping on decoupled architectures such as Ascend NPUs.

\subsection{Problem Formulation}
\label{sec:formulation}

Consider a linear layer with weight matrix $W \in \mathbb{Z}_8^{m \times n}$ (quantized to INT8 offline with per-channel scale $s_W \in \mathbb{R}^m$) and activation vector $x \in \mathbb{R}^n$ (in BF16 format during inference). The standard dequantization-based approach computes:
\begin{equation}
y = \text{dequant}(W) \cdot x = (s_W \odot W_{\text{int8}}) \cdot x,
\end{equation}
where the weight matrix is first dequantized from INT8 to BF16 via per-channel scaling, then multiplied with the BF16 activation. This dequantization step is the bottleneck we aim to eliminate.

MSD takes a different approach: instead of dequantizing $W$ from INT8 to BF16, we decompose the BF16 activation $x$ into $K$ INT8 components:
\begin{equation}
x \approx \sum_{k=1}^{K} \alpha_k \cdot x^{(k)}, \quad x^{(k)} \in \mathbb{Z}_8^n,
\end{equation}
where $\alpha_k \in \mathbb{R}$ are learned or computed scaling coefficients. The output is then computed as:
\begin{equation}
y = \sum_{k=1}^{K} \alpha_k \cdot (W x^{(k)}).
\end{equation}

Each $W x^{(k)}$ is a native INT8$\times$INT8 GEMM, executable directly on hardware tensor cores without dequantization. Since the weight scale $s_W$ is per-output-channel (i.e., each row of $W$ shares a single scale), it can be applied \emph{after} the GEMM and reconstruction:
\begin{equation}
y = s_W \odot \left(\sum_{k=1}^{K} \alpha_k \cdot (W_{\text{int8}} x^{(k)})\right),
\end{equation}
where $\odot$ denotes row-wise scaling. This is a lightweight $O(m)$ Vector operation, analogous to how V's per-channel scale is applied after the $PV$ GEMM in attention (Section~\ref{sec:attention}). In practice, we find $K=2$ provides an excellent trade-off between accuracy and computational cost.

\subsection{Two-Pass Decomposition Algorithm}
\label{sec:decomposition}

For $K=2$, we use a two-pass decomposition analogous to multi-grid correction in numerical analysis. The algorithm proceeds as follows:

\textbf{Pass 1: Coarse-Scale Quantization.} Compute the primary scale and quantized activation:
\begin{align}
\alpha &= \frac{\|x\|_\infty}{127}, \label{eq:alpha} \\
x^{(1)} &= \text{clamp}\left(\text{round}\left(\frac{x}{\alpha}\right), -128, 127\right). \label{eq:x1}
\end{align}

\textbf{Pass 2: Fine-Scale Residual.} Compute the residual. Since the quantization error of Pass 1 is bounded in $(-0.5\alpha, 0.5\alpha)$, we directly use $2 \times 127 = 254$ as the secondary scale without computing max:
\begin{align}
r &= x - \alpha \cdot x^{(1)}, \label{eq:residual} \\
\beta &= \frac{\alpha}{254}, \label{eq:beta} \\
x^{(2)} &= \text{clamp}\left(\text{round}\left(\frac{r}{\beta}\right), -128, 127\right). \label{eq:x2}
\end{align}

The final approximation is $x \approx \alpha \cdot x^{(1)} + \beta \cdot x^{(2)}$. The decomposition is performed on-the-fly for each activation vector during inference, with negligible overhead compared to the subsequent GEMM operations.

\subsubsection{MXFP4 Instantiation}
\label{sec:mxfp4-instantiation}

When weights are in MXFP4 format (W4A16), the MSD framework adapts to the Microscaling (MX) specification~\cite{microscaling}, which imposes two key constraints: (1) the shared scale per 32-element block must be a power of two (E8M0 format), and (2) quantized values use the FP4 E1M2 format with representable positive values $\{0, 0.25, 0.5, 0.75, 1.0, 1.25, 1.5, 1.75\}$ and a \emph{uniform step size of 0.25}.

These constraints are not merely mathematical limitations---they are essential for \textbf{hardware realizability}. The E8M0 power-of-two constraint ensures that scaling operations ($x / \alpha$ and $r / \beta$) reduce to simple exponent adjustments, which can be implemented directly in hardware as bit-shifts on floating-point exponents without multipliers or dividers. Our choice of $\alpha$ and $\beta$ as powers of two is therefore deliberate: it ensures the entire decomposition pipeline---scale, quantize, compute residual, re-scale---maps onto native MX hardware instructions without any software-emulated scaling. This is why we adopt the E8M0-constrained $\alpha = 2^{\lceil \log_2(M_b / 1.859375) \rceil}$ and $\beta = \alpha / 2^4$ rather than the mathematically optimal (but non-power-of-two) scales that would arise from unconstrained optimization.

These constraints change the decomposition design in three ways:

\textbf{1. $\alpha$ selection with E8M0 power-of-two constraint.} Since $\alpha$ must be a power of two, we select:
\begin{equation}
\alpha = 2^{\lceil \log_2(M_b / 1.859375) \rceil},
\label{eq:alpha-mxfp4}
\end{equation}
where $M_b = \|x_b\|_\infty$ is the maximum absolute value in the 32-element block, and $1.859375 = 1.75 \times 17/16$. The factor 1.859375 \emph{deliberately exceeds} FP4's maximum representable value of 1.75. Elements with $|x_i / \alpha| \in (1.75, 1.859375]$ are clipped to $\pm 1.75$ via round-to-nearest, and their residual is captured by Pass 2. This relaxation allows $\alpha$ to be halved more often, improving Pass 1 quantization granularity for all 32 elements in the block.

\textbf{2. $\beta$ derived directly from $\alpha$ (no max-reduction).} Unlike the INT8 case where $\beta = \alpha / 254$ is derived from the INT8 range, for MXFP4 we set:
\begin{equation}
\beta = \frac{\alpha}{16} = \frac{\alpha}{2^4}.
\label{eq:beta-mxfp4}
\end{equation}
Since $\beta$ is also a power of two, it satisfies the E8M0 constraint. Crucially, $\beta$ is computed from $\alpha$ alone---no max-reduction over the residual is needed. This is a direct consequence of the MSD framework: the Pass 1 residual bound is known analytically ($\|r\|_\infty \leq \alpha/8$), so the Pass 2 scale can be set to cover this range without examining the data.

\textbf{3. Truncation analysis.} The scaled residual satisfies $|r_i / \beta| \leq 2$, which exceeds FP4's range of $[-1.75, 1.75]$. Approximately 12.5\% of elements fall in $(1.75, 2]$ and are clipped to $\pm 1.75$ via round-to-nearest. This is an \emph{intentional} trade-off: the 87.5\% of elements that are normally quantized achieve error $\leq \alpha/128$, while the clipped 12.5\% have error $\leq \alpha/64$ (Theorem~\ref{thm:error-mxfp4}).

Algorithm~\ref{alg:msd-mxfp4} summarizes the MXFP4 decomposition procedure per 32-element block.

\begin{algorithm}[h]
\caption{MSD MXFP4 Decomposition (per 32-element block)}
\label{alg:msd-mxfp4}
\begin{algorithmic}[1]
\Require Block $x_b \in \mathbb{R}^{32}$ (BF16), Weight block $W_b$ (MXFP4)
\Ensure Decomposed components $q_1, q_2 \in \text{FP4}^{32}$, scales $\alpha, \beta \in \text{E8M0}$
\State $M_b \gets \|x_b\|_\infty$
\State $\alpha \gets 2^{\lceil \log_2(M_b / 1.859375) \rceil}$ \Comment{E8M0 scale (power of two)}
\State $s \gets x_b / \alpha$ \Comment{Scale to FP4 range}
\State $q_1 \gets \text{round\_to\_FP4}(s)$ \Comment{Round-to-nearest on E1M2 grid; $|s|>1.75$ maps to $\pm 1.75$}
\State $r \gets x_b - \alpha \cdot q_1$ \Comment{Residual; $\|r\|_\infty \leq \alpha/8$}
\State $\beta \gets \alpha / 16$ \Comment{E8M0 scale; no max-reduction needed}
\State $q_2 \gets \text{round\_to\_FP4}(r / \beta)$ \Comment{12.5\% elements clipped to $\pm 1.75$; error $\leq \alpha/64$}
\State \Return $q_1, q_2, \alpha, \beta$
\end{algorithmic}
\end{algorithm}

\subsection{Optimization of Scaling Coefficients}
\label{sec:optimization}

The two-pass decomposition minimizes the $L_\infty$ reconstruction error in a greedy manner. Alternatively, one can formulate the optimal decomposition as a constrained least-squares problem:
\begin{equation}
\min_{\alpha, \beta, x^{(1)}, x^{(2)}} \|x - (\alpha x^{(1)} + \beta x^{(2)})\|_2^2 \quad \text{s.t.} \quad x^{(i)} \in \mathbb{Z}_8^n.
\end{equation}

This integer optimization is NP-hard in general. In practice, we find that the greedy two-pass algorithm achieves near-optimal results with $O(n)$ complexity, making it suitable for online inference. For offline calibration scenarios, grid search over candidate $(\alpha, \beta)$ pairs can provide marginal improvements.

\textbf{Tighter bounds via fractional scaling.} A simple refinement is to use $\alpha = \|x\|_\infty / 127.49$ instead of $\|x\|_\infty / 127$ (and correspondingly $\beta = \alpha / 254.98$). The rationale is as follows: with $\alpha = M/127$, the extremal element satisfies $|x_i / \alpha| = 127$ exactly, so the rounding error is zero for that element but up to $\alpha/2$ for others. With $\alpha = M/127.49$, the extremal element maps to $127.49$, which rounds to $127$ with residual $0.49\alpha < 0.5\alpha$. This tightens the worst-case residual bound from $\alpha/2$ to $0.49\alpha$, and the improvement propagates through subsequent passes:
\begin{equation}
\text{Error bound:} \quad \frac{M}{127.49 \times 254.98 \times 2} \approx \frac{M}{65015} \quad \text{vs.} \quad \frac{M}{127 \times 254 \times 2} = \frac{M}{64516}.
\end{equation}
While the improvement is modest ($\sim$0.8\%), it is essentially free---requiring no additional computation, only a change in the scaling constants.

\textbf{MXFP4 $\alpha$ relaxation optimization.} For the MXFP4 instantiation, a different form of scaling optimization yields substantial gains. The key insight is to relax $\alpha$'s upper bound beyond FP4's maximum representable value (1.75). Table~\ref{tab:mxfp4-evolution} shows the progressive improvement from three design iterations:

\begin{table}[htbp]
\centering
\setlength{\tabcolsep}{4pt} 
\begin{tabular}{@{} l c c c c c c @{}}
\toprule
\textbf{Config.} & \textbf{$\alpha$ Bound} & \textbf{$\beta$} & \textbf{Clip\%} & \textbf{Eff. Bits} & \textbf{L2 Error} & \textbf{vs.\ MXFP8} \\
\midrule
v1 (orig.) & 1.75 & $\alpha/8$ & 0\% & 5.79 & 0.0182 & $1.5\times$ \\
v2 (finer $\beta$) & 1.75 & $\alpha/16$ & $\sim$12.5\% & 6.55 & 0.0107 & $2.5\times$ \\
\textbf{v3 (opt.)} & \textbf{1.859375} & \textbf{$\alpha$/16} & $\sim$\textbf{12.5\%} & \textbf{6.65} & \textbf{0.0101} & $\mathbf{2.6\times}$ \\
\bottomrule
\end{tabular}
\vspace{0.15cm} 
\caption{MXFP4 decomposition design evolution (2048$\times$2048 GEMM, Gaussian activation, MXFP4 weight).}
\label{tab:mxfp4-evolution}
\end{table}

The v1$\to$v2 improvement comes from using a finer $\beta$: since $\beta = \alpha/8$ leaves residual headroom (the maximum $|r/\beta| = 2$, but only 1.75 is representable), switching to $\beta = \alpha/16$ halves the Pass 2 quantization step at the cost of $\sim$12.5\% clipping. The v2$\to$v3 improvement comes from relaxing $\alpha$'s upper bound to 1.859375: when $\max|block|/\alpha$ falls in $(1.75, 1.859375]$, $\alpha$ can be halved, doubling Pass 1 precision for the entire block. The overflow is exactly captured by the $\beta = \alpha/16$ residual pass (Theorem~\ref{thm:error-mxfp4}).

\subsection{Extension to $K > 2$ Scales}

While we use $K=2$ in this paper, the MSD framework is general and supports arbitrary decomposition granularity $K$. The key insight is that the multi-scale decomposition can be applied iteratively: after the second pass, we can continue decomposing the residual to obtain $x^{(3)}, x^{(4)}, \dots$

For the BF16 + INT8 combination studied in this paper, we find $K=2$ provides sufficient accuracy—indeed, it achieves lower error than traditional dequantization-based approaches while maintaining comparable effective compute time. Adding more scales would increase the number of GEMMs without meaningful accuracy gains.

However, $K > 2$ becomes valuable when the precision gap between activation and weight is larger. For example:
\begin{itemize}
    \item \textbf{BF16 $\to$ INT4:} When decomposing BF16 activations to INT4 components, two scales may not fully capture the dynamic range. We can use $K=3$ or $K=4$ to progressively refine the residual.
    \item \textbf{FP16 $\to$ INT4:} Similar to BF16, but with different dynamic range characteristics.
    \item \textbf{Mixed-precision scenarios:} For emerging formats like FP8 or MXFP4, the optimal $K$ depends on the specific precision combination.
\end{itemize}

The general $K$-scale MSD algorithm follows the same pattern: each additional scale $\gamma_i$ can be computed directly from the previous scale without explicit max computation (since the residual error after $i-1$ passes is bounded by $\alpha / (2 \cdot 254^{i-1})$).

\textbf{Trade-off:} Increasing $K$ improves approximation accuracy but requires more GEMM operations. On accelerators with strong INT4 throughput (typically $4\times$ BF16), this trade-off can be \emph{better than break-even}---yielding actual Cube-side speedup, as we analyze below.

\subsubsection{The BF16 + MXFP4 Case: MSD-MXFP4 for W4A16}

The MXFP4 weight quantization scenario (W4A16) deserves special attention. In the MX ecosystem, the natural baseline for activation quantization is single-pass MXFP8 (5.24 effective bits). MSD with $K=2$ MXFP4 passes achieves $\sim$6.6 effective bits (Theorem~\ref{thm:error-mxfp4}), surpassing MXFP8 by 1.4 bits---using only two 4-bit passes rather than one 8-bit pass.

The error bound for MSD-MXFP4 is $\alpha/64$ per 32-element block (Theorem~\ref{thm:error-mxfp4}), which is a per-block guarantee rather than the per-vector guarantee of the INT8 variant. This reflects the MX specification's per-block scaling: each 32-element block has its own E8M0 scale $\alpha$, and the error bound scales accordingly.

\textbf{Effective compute time.} On modern accelerators (e.g., NVIDIA Blackwell, Ascend 910B), FP4 GEMM throughput is approximately $4\times$ that of BF16, while FP8 throughput is $2\times$. The compute time comparison is:

\begin{table}[h]
\centering
\caption{Effective compute time: MXFP8 baseline vs. MSD-MXFP4 ($K=2$)}
\label{tab:mxfp4-compute}
\begin{tabular}{lccc}
\toprule
Method & Raw FLOPs & Throughput & Effective Time \\
\midrule
MXFP8$\times$MXFP4 (dequant) & $2mn$ (FP8) & $2\times$ & $1.0mn$ \\
\textbf{MSD-MXFP4 ($2\times$FP4)} & $4mn$ (FP4) & $4\times$ & $\mathbf{1.0mn}$ \\
\bottomrule
\end{tabular}
\end{table}

Two FP4 GEMMs at $4\times$ throughput yield the same effective Cube time as one FP8 GEMM at $2\times$ throughput. MSD-MXFP4 therefore maintains comparable compute time while achieving 1.4 more effective bits of activation precision, removing weight dequantization from the critical path, and providing a provable per-block error bound.

This makes the MXFP4 scenario uniquely favorable for MSD: the lower weight precision (4-bit) makes activation precision more critical, and MSD's two-pass decomposition fills this gap by surpassing the 8-bit activation baseline. Combined with the growing adoption of W4 quantization (GPTQ, AWQ, QuIP\#) and the MX standard~\cite{microscaling}, MSD-MXFP4 is a practical approach for next-generation W4A16 inference engines.

\subsection{Hardware Mapping}
\label{sec:hardware}

The MSD workflow maps efficiently to architectures with decoupled compute units. We use Ascend 910B as a concrete example:

\textbf{Step 1: Decomposition (Vector Core).} The activation vector $x$ is loaded into L0 buffer. Vector cores compute $\alpha$, $x^{(1)}$, the residual $r$, $\beta$, and $x^{(2)}$ via parallel element-wise operations. This step is memory-bandwidth-bound and completes quickly.

\textbf{Step 2: Dual GEMM (Cube Core).} Both $x^{(1)}$ and $x^{(2)}$ are fed to the Cube core for INT8$\times$INT8 GEMM with weight matrix $W$. Modern tensor cores (Ascend Cube, NVIDIA Tensor Cores) support native INT8$\times$INT8$\to$INT32 accumulation at full throughput.

\textbf{Step 3: Reconstruction (Vector Core).} The two partial outputs $y^{(1)} = W_{\text{int8}} x^{(1)}$ and $y^{(2)} = W_{\text{int8}} x^{(2)}$ are scaled by $\alpha$ and $\beta$ respectively, summed, and then multiplied by the per-channel weight scale $s_W$ to produce the final BF16 output $y = s_W \odot (\alpha y^{(1)} + \beta y^{(2)})$.

To maximize throughput, we implement a fused tiled kernel where the weight tile remains resident on-chip across both MSD passes (the \emph{resident-tile} condition), decomposition, GEMM, and reconstruction overlap via double buffering, and partial results are not materialized to HBM. When the tile cannot remain resident, the implementation falls back to a conservative two-read model with approximately $\sim 1.5\times$ traffic reduction in the dominant term rather than $\sim 3\times$.

Under the resident-tile model, MSD reduces HBM traffic from $3mn + 2bn + 2bm$ (dequant) to $mn + 4bn + 2bm$---a $\sim 3\times$ reduction in the dominant term since $b \ll m,n$. Since INT8 GEMM throughput is $2\times$ that of BF16, MSD's $4bmn$ INT8 FLOPs have comparable effective Cube time to the dequant baseline's $2bmn$ BF16 FLOPs. For MXFP4, two FP4 GEMMs at $4\times$ throughput yield the same effective Cube time as one FP8 GEMM at $2\times$ throughput. A detailed cost analysis with latency models is provided in Section~\ref{sec:analysis}.

\subsection{Vector Compute Overhead}

The MSD decomposition and reconstruction involve only $O(bn+bm)$ Vector FLOPs, compared to $O(mn)$ for dequantization (Table~\ref{tab:vector-ops}). For typical transformer layers ($d=4096$) with small $b$, Vector ops are $<0.1\%$ of total FLOPs.

\begin{table}[h]
\centering
\caption{Vector-side compute operations per layer}
\label{tab:vector-ops}
\begin{tabular}{lcc}
\toprule
Operation & FLOPs & Description \\
\midrule
Decomposition (Pass 1) & $3n$ & abs, max, divide, round, clamp \\
Decomposition (Pass 2) & $5n$ & residual, divide, round, clamp \\
Reconstruction & $2m$ & scale multiply, add, cast \\
\midrule
Total Vector FLOPs & $O(n + m)$ & \\
\end{tabular}
\end{table}

\subsection{MSD for Mixed-Precision Configurations}

MSD is a general framework that applies to any combination of activation and weight precision:

\begin{table}[h]
\centering
\caption{MSD applicability to various precision configurations. $K$ denotes the number of decomposition scales. Eff.\ Bits and Error Bound are for the $K$ shown.}
\label{tab:precision-configs}
\begin{tabular}{lclccc}
\toprule
Activation & Weight & MSD Decomposition & $K$ & Eff.\ Bits & Error Bound \\
\midrule
BF16 & INT8 & BF16 $\to$ INT8 + INT8 & 2 & $\sim$16 & $M/64516$ \\
BF16 & MXFP4 & BF16 $\to$ MXFP4 + MXFP4 & 2 & $\sim$6.6 & $\alpha/64$ \\
BF16 & FP8 & BF16 $\to$ FP8 + FP8 & 2 & $\sim$16 & TBD \\
BF16 & INT4 & BF16 $\to$ INT4 + INT4 + INT4 & 3 & $\sim$11 & TBD \\
\midrule
\multicolumn{6}{l}{\emph{Baselines for comparison:}} \\
\midrule
BF16 & INT8 & Dequant (BF16$\times$BF16) & --- & $\sim$8 & --- \\
BF16 & MXFP4 & MXFP8 activation & --- & $\sim$5.24 & --- \\
\bottomrule
\end{tabular}
\end{table}

\textbf{Key insight:} MSD shifts the decomposition from weights to activations, enabling native low-precision GEMM regardless of the weight format. For both INT8 and MXFP4 weight formats, MSD's $K=2$ decomposition surpasses the respective single-pass activation baselines while maintaining comparable effective compute time.

\subsection{Decode vs. Prefill: When to Use MSD}

MSD is designed for \textbf{Decode-heavy inference workloads} where batch sizes are small ($b \ll m, n$) and latency per token is critical. In decode, the additional MSD GEMM is absorbed by INT8's $2\times$ throughput, and the Vector-side decomposition/merging cost is $O(bn+bm) \ll O(bmn)$. In prefill with large batch sizes, the extra MSD GEMM grows linearly with $b$ and the dequantization cost is amortized, so MSD is not recommended. Detailed analysis for the attention case is provided in Section~\ref{sec:attention}, and the operator coverage policy in Section~\ref{sec:discussion}.

\subsection{Fused Tiled Kernel Realization}
\label{sec:fused-kernel}

The performance claims in this paper depend on implementing MSD as a \emph{fused tiled kernel}, not as two standalone GEMM invocations. If the two MSD GEMM passes were executed as independent kernels, the weight/KV data would be read twice from HBM, partial outputs would be materialized, and kernel launch overhead would erode the benefits. A fused tiled kernel avoids these pitfalls through the following design principles:

\begin{enumerate}
    \item \textbf{Resident weight/KV tile.} Each weight or KV tile is loaded from HBM \emph{once} into on-chip buffer and consumed by both MSD passes before eviction. The tile must satisfy $m_t k_t b_w \leq C_{\text{tile}}$, where $C_{\text{tile}}$ is the available on-chip capacity. For attention decode with $B_c = 64$, $d = 128$: the KV tile is only 8\,KB---easily resident.
    \item \textbf{Online activation decomposition.} The activation components $x^{(1)}, x^{(2)}$ (or $Q^{(1)}, Q^{(2)}$, $P^{(1)}, P^{(2)}$ in attention) are generated on-the-fly per tile, not materialized to HBM.
    \item \textbf{In-register/streaming partial results.} The partial outputs $y^{(1)}, y^{(2)}$ from the two GEMM passes are scaled, summed, and accumulated into the final output buffer via FixPipe (on Ascend) or register-level operations---without intermediate HBM writes.
    \item \textbf{Single final writeback.} Only the reconstructed output $y = s_W \odot (\alpha y^{(1)} + \beta y^{(2)})$ is written to HBM.
\end{enumerate}

Figure~\ref{fig:data-path} contrasts the data paths of dequantization-based and MSD-based execution.

\begin{figure}[h]
\centering
\begin{tikzpicture}[
    box/.style={draw, rounded corners, minimum width=2.2cm, minimum height=0.7cm, align=center, font=\small},
    arrow/.style={-{Stealth[length=2.5mm]}, thick},
    label/.style={font=\small\bfseries, anchor=south},
    every node/.style={font=\small}
]

\node[label] at (-1.5, 4.2) {Dequant Path};

\node[box, fill=blue!10] (hbm1) at (0, 3.5) {HBM\\INT8 W/KV};
\node[box, fill=orange!15] (vec1) at (3.2, 3.5) {Vector\\dequant\\INT8$\to$BF16};
\node[box, fill=blue!10] (hbm2) at (6.4, 3.5) {HBM\\BF16 W/KV};
\node[box, fill=green!15] (cube1) at (9.6, 3.5) {Cube\\BF16 GEMM};

\draw[arrow] (hbm1) -- (vec1);
\draw[arrow, red, thick] (vec1) -- node[above, font=\scriptsize, text=red] {round-trip} (hbm2);
\draw[arrow] (hbm2) -- (cube1);

\node[label] at (-1.5, 1.5) {MSD Fused};

\node[box, fill=blue!10] (hbm3) at (0, 0.8) {HBM\\INT8 W/KV};
\node[box, fill=green!30] (tile) at (3.2, 0.8) {On-chip\\resident tile};
\node[box, fill=yellow!15] (decomp) at (3.2, -0.4) {Decompose\\$x \to x^{(1)}, x^{(2)}$};
\node[box, fill=green!15] (cube2) at (6.4, 0.8) {Cube Pass 1\\$W x^{(1)}$};
\node[box, fill=green!15] (cube3) at (6.4, -0.4) {Cube Pass 2\\$W x^{(2)}$};
\node[box, fill=gray!15] (merge) at (9.6, 0.2) {Merge + write\\$\alpha y^{(1)}\!+\!\beta y^{(2)}$};

\draw[arrow] (hbm3) -- (tile);
\draw[arrow] (tile) -- (cube2);
\draw[arrow] (tile) -- (cube3);
\draw[arrow] (decomp) -- (cube2);
\draw[arrow] (decomp) -- (cube3);
\draw[arrow] (cube2) -- (merge);
\draw[arrow] (cube3) -- (merge);

\node[font=\scriptsize, text=red, anchor=west] at (11, 3.5) {dominant $\sim 3mn$};
\node[font=\scriptsize, text=blue!70!black, anchor=west] at (11, 0.2) {dominant $\sim mn$};

\end{tikzpicture}
\caption{Data path comparison. \textbf{Top:} Dequantization-based execution requires INT8$\to$BF16 conversion on Vector cores followed by a round-trip through HBM before Cube GEMM, yielding dominant HBM traffic of $\sim 3mn$ bytes. \textbf{Bottom:} MSD fused tiled execution loads the weight/KV tile once into on-chip buffer, decomposes activations on-the-fly, and runs two low-precision GEMM passes against the same resident tile. Partial results are merged on-chip; only the final output is written to HBM, yielding dominant HBM traffic of $\sim mn$ bytes. When the left matrix has few rows (e.g., decode with $b \ll m, n$), the two GEMM passes can be further fused into a single GEMM by concatenating $X^{(1)}$ and $X^{(2)}$ along the row dimension (see text).}
\label{fig:data-path}
\end{figure}
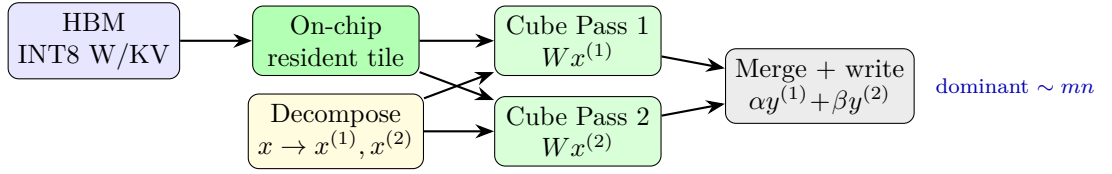

\textbf{GEMM pass fusion for small-batch decode.} The MSD decomposition produces two activation components $X^{(1)}, X^{(2)} \in \mathbb{R}^{b \times n}$, requiring two separate GEMM calls: $W X^{(1)}$ and $W X^{(2)}$. However, when $b$ is small---as is typical in decode where $b \ll m,n$---the two GEMM passes can be \emph{fused into a single GEMM call} by concatenating the components along the row dimension:
\begin{equation}
\begin{bmatrix} \alpha Y^{(1)} \\ \beta Y^{(2)} \end{bmatrix} = \begin{bmatrix} \alpha X^{(1)} \\ \beta X^{(2)} \end{bmatrix} \cdot W^\top,
\end{equation}
where $\begin{bmatrix} \alpha X^{(1)} \\ \beta X^{(2)} \end{bmatrix} \in \mathbb{R}^{2b \times n}$ and the result is a $2b \times m$ matrix from which $\alpha Y^{(1)}$ and $\beta Y^{(2)}$ are extracted and summed. This reduces two kernel launches to one, eliminates inter-kernel synchronization, and improves Cube utilization. When $b$ is large (e.g., large-batch prefill), this concatenation may exceed on-chip capacity, and the two GEMM passes must be computed separately.

When the resident-tile condition cannot be met (e.g., very large weight matrices without sufficient on-chip capacity), the implementation falls back to the conservative two-read model, reducing the traffic benefit from $\sim 3\times$ to $\sim 1.5\times$ in the dominant term while still avoiding the dequantization round-trip.

\textbf{Attention decode example.} The strongest application of the fused tiled kernel is attention decode, where KV tiles are small and the memory-bound regime makes HBM savings most impactful. For each KV block: (1) load $K_t, V_t$ once into on-chip buffer; (2) decompose $Q \to Q^{(1)}, Q^{(2)}$; (3) compute dual GEMMs $Q^{(i)} K_t^\top$ using the same resident $K_t$; (4) merge and apply online softmax; (5) decompose $P_t \to P_t^{(1)}, P_t^{(2)}$; (6) compute dual GEMMs $P_t^{(i)} V_t$ using the same resident $V_t$; (7) merge and update running output. Only the final $O$ is written to HBM. See Algorithm~\ref{alg:msd-attention} in Section~\ref{sec:attention} for the complete procedure.

Linear and grouped GEMM kernels follow the same tile-level principle: each resident weight tile is loaded once, consumed by multiple activation components, and merged before final writeback.

\subsection{Pseudocode}

Algorithm~\ref{alg:msd} summarizes the complete MSD inference procedure for a single linear layer.

\begin{algorithm}[h]
\caption{MSD Inference for a Single Linear Layer}
\label{alg:msd}
\begin{algorithmic}[1]
\Require Activation $x \in \mathbb{R}^n$ (BF16), Weight $W_{\text{int8}} \in \mathbb{Z}_8^{m \times n}$ (INT8), Per-channel scale $s_W \in \mathbb{R}^m$
\Ensure Output $y \in \mathbb{R}^m$ (BF16)
\State $\alpha \gets \|x\|_\infty / 127$
\State $x^{(1)} \gets \text{clamp}(\text{round}(x / \alpha), -128, 127)$ \Comment{INT8}
\State $r \gets x - \alpha \cdot x^{(1)}$ \Comment{BF16 residual}
\State $\beta \gets \alpha / 254$ \Comment{Directly use 254 as scale, no max needed}
\State $x^{(2)} \gets \text{clamp}(\text{round}(r / \beta), -128, 127)$ \Comment{INT8}
\State $y^{(1)} \gets W_{\text{int8}} \cdot x^{(1)}$ \Comment{Native INT8$\times$INT8 GEMM (Cube)}
\State $y^{(2)} \gets W_{\text{int8}} \cdot x^{(2)}$ \Comment{Native INT8$\times$INT8 GEMM (Cube)}
\State $y \gets s_W \odot (\alpha \cdot y^{(1)} + \beta \cdot y^{(2)})$ \Comment{Reconstruct + apply weight scale (Vector)}
\State \Return $y$
\end{algorithmic}
\end{algorithm}

\section{MSD for Multi-Head Attention}
\label{sec:attention}

This section details how Multi-Scale Dequant (MSD) applies to the attention computation in transformers, following the notation of FlashAttention~\cite{flashattn}.

\subsection{Standard Attention Formulation}

Given queries $Q \in \mathbb{R}^{N \times d}$, keys $K \in \mathbb{R}^{M \times d}$, and values $V \in \mathbb{R}^{M \times d}$, the attention output is:
\begin{equation}
O = \text{Attention}(Q, K, V) = \text{softmax}\left(\frac{Q K^\top}{\sqrt{d}}\right) V,
\end{equation}
where $N$ is the query sequence length, $M$ is the key/value sequence length, and $d$ is the head dimension.

In FlashAttention, the computation uses tiling to reduce memory IO:
\begin{enumerate}
    \item Compute $S = Q K^\top \in \mathbb{R}^{N \times M}$ (attention scores)
    \item Compute $P = \text{softmax}(S) \in \mathbb{R}^{N \times M}$ (attention weights)
    \item Compute $O = P V \in \mathbb{R}^{N \times d}$ (output)
\end{enumerate}

In quantized FlashAttention, $K$ and $V$ are stored in INT8 format with per-channel scales $s_K, s_V \in \mathbb{R}^d$ (KVCache), while $Q$ is typically in BF16. The dequantization bottleneck arises when converting $K$ and $V$ from INT8 to BF16 (via $K_{\text{bf16}} = K_{\text{int8}} \odot s_K$) before the GEMM operations.

\subsection{MSD for Attention: Leveraging Online Softmax}

Standard FlashAttention computes attention scores in blocks, using online softmax which maintains the maximum value for each row to ensure numerical stability. Specifically, during the tiling-based computation, FlashAttention tracks:
\begin{equation}
m_i = \max_j S_{ij}
\end{equation}
for each block $i$, which is already computed as part of the softmax rescaling.

MSD leverages this existing $m_i$ value for decomposing the attention weight matrix $P$. After the softmax produces $P \in \mathbb{R}^{N \times M}$ in BF16, we decompose it similarly to activations:

\textbf{Step 1: Absorb K scale into Q, then decompose.}
Since $K_{\text{real}} = K_{\text{int8}} \odot s_K$ where $s_K \in \mathbb{R}^d$ is per-channel, we have:
\begin{equation}
S = Q K_{\text{real}}^\top / \sqrt{d} = (Q \odot s_K) K_{\text{int8}}^\top / \sqrt{d}.
\end{equation}
We first absorb the K scale into Q: $\tilde{Q} = Q \odot s_K$, then apply MSD decomposition to $\tilde{Q}$:
$\alpha_Q = \|\tilde{Q}\|_\infty / 127$ and $\beta_Q = \alpha_Q / 254$.

\textbf{Step 2: $S = \tilde{Q} K^\top$ with dual GEMM.}
$\tilde{Q}$ is decomposed while $K$ remains in INT8. We compute:
\begin{align}
S^{(1)} &= \tilde{Q}^{(1)} \cdot K_{\text{int8}}^\top \in \mathbb{R}^{N \times M}, \\
S^{(2)} &= \tilde{Q}^{(2)} \cdot K_{\text{int8}}^\top \in \mathbb{R}^{N \times M}, \\
S &= (\alpha_Q \cdot S^{(1)} + \beta_Q \cdot S^{(2)}) / \sqrt{d}.
\end{align}

Each $\tilde{Q}^{(i)} K_{\text{int8}}^\top$ is a native INT8$\times$INT8 GEMM. The scaling and summation are element-wise Vector operations.

\textbf{Step 3: Online Softmax with MSD fusion.}
During online softmax, we first compute $P = \exp(S - m)$ where $m = \max(S)$ is the row-wise maximum already tracked for numerical stability. Note that $P$ here is the \emph{unnormalized} softmax numerator (the denominator $\ell = \sum_j P_{ij}$ is applied later); $P \in [0, 1]^{N \times M}$ consists of non-negative values. We apply standard MSD decomposition to $P$:
\begin{align}
\alpha_P &= \frac{\|P\|_\infty}{127}, \\
P^{(1)} &= \text{clamp}\left(\text{round}\left(\frac{P}{\alpha_P}\right), -128, 127\right), \\
r_P &= P - \alpha_P \cdot P^{(1)}, \\
\beta_P &= \frac{\alpha_P}{254}, \\
P^{(2)} &= \text{clamp}\left(\text{round}\left(\frac{r_P}{\beta_P}\right), -128, 127\right).
\end{align}

The key observation is that $\|P\|_\infty = \max(\exp(S - m)) = 1$ (since the maximum element of $S - m$ is zero), so $\alpha_P = 1/127$ is a \emph{constant} that requires no additional computation. This makes the MSD decomposition of $P$ essentially free in terms of the max-finding step.

\textbf{Step 4: PV with dual GEMM.}
$P$ is decomposed while $V$ remains in INT8. Since $V_{\text{real}} = V_{\text{int8}} \odot s_V$ where $s_V$ is per-channel, we can apply the V scale \emph{after} the GEMM:
\begin{align}
O^{(1)} &= P^{(1)} \cdot V_{\text{int8}} \in \mathbb{R}^{N \times d}, \\
O^{(2)} &= P^{(2)} \cdot V_{\text{int8}} \in \mathbb{R}^{N \times d}, \\
O &= (\alpha_P \cdot O^{(1)} + \beta_P \cdot O^{(2)}) \odot s_V.
\end{align}

Each $P^{(i)} V_{\text{int8}}$ is a native INT8$\times$INT8 GEMM. The per-channel scale $s_V$ is applied element-wise after reconstruction, which is a lightweight Vector operation.

\begin{algorithm}[h]
\caption{MSD Attention (per head)}
\label{alg:msd-attention}
\begin{algorithmic}[1]
\Require $Q \in \mathbb{R}^{N \times d}$ (BF16), $K \in \mathbb{Z}_8^{M \times d}$ (INT8), $V \in \mathbb{Z}_8^{M \times d}$ (INT8), per-channel scales $s_K, s_V \in \mathbb{R}^d$
\Ensure $O \in \mathbb{R}^{N \times d}$ (BF16)
\Statex \textbf{Absorb K scale into Q, then decompose:}
\State $\tilde{Q} \gets Q \odot s_K$ \Comment{Per-channel: $\tilde{Q}_{ik} = Q_{ik} \cdot s_{K,k}$}
\State $(\alpha_Q, \beta_Q, \tilde{Q}^{(1)}, \tilde{Q}^{(2)}) \gets \text{MSD-Decompose}(\tilde{Q})$
\Statex \textbf{$S = \tilde{Q} K^\top$ with dual GEMM:}
\State $S^{(1)} \gets \tilde{Q}^{(1)} \cdot K^\top$ \Comment{INT8$\times$INT8 GEMM}
\State $S^{(2)} \gets \tilde{Q}^{(2)} \cdot K^\top$ \Comment{INT8$\times$INT8 GEMM}
\State $S \gets (\alpha_Q \cdot S^{(1)} + \beta_Q \cdot S^{(2)}) / \sqrt{d}$
\Statex \textbf{Softmax + MSD decompose $P$:}
\State $m \gets \max(S, \text{row-wise})$ \Comment{Online softmax}
\State $P \gets \exp(S - m)$ \Comment{$\|P\|_\infty = 1$}
\State $\alpha_P \gets 1/127$, $\beta_P \gets \alpha_P / 254$
\State $(P^{(1)}, P^{(2)}) \gets \text{MSD-Decompose}(P, \alpha_P, \beta_P)$
\Statex \textbf{$O = P V$ with dual GEMM + V scale:}
\State $O^{(1)} \gets P^{(1)} \cdot V$ \Comment{INT8$\times$INT8 GEMM}
\State $O^{(2)} \gets P^{(2)} \cdot V$ \Comment{INT8$\times$INT8 GEMM}
\State $O \gets (\alpha_P \cdot O^{(1)} + \beta_P \cdot O^{(2)}) \odot s_V$ \Comment{Apply V scale after GEMM}
\State $O \gets O / \ell$ \Comment{$\ell = \sum P$ (softmax normalizer)}
\State \Return $O$
\end{algorithmic}
\end{algorithm}

Algorithm~\ref{alg:msd-attention} summarizes the full MSD attention procedure.

Key observations: (1) K's per-channel scale $s_K$ is absorbed into $Q$ before MSD decomposition, so the $\tilde{Q}^{(i)} K_{\text{int8}}^\top$ GEMMs are pure INT8$\times$INT8. V's per-channel scale $s_V$ is applied after the $P^{(i)} V_{\text{int8}}$ GEMMs. (2) Since $\|P\|_\infty = 1$ (from the softmax max subtraction), $\alpha_P = 1/127$ is a constant---no max computation is needed for P's decomposition.

\subsection{Integration with FlashAttention Tiling}

MSD integrates seamlessly with FlashAttention's tile-based computation:

\begin{enumerate}
    \item \textbf{Load Q tile:} Absorb K's per-channel scale: $\tilde{Q} = Q \odot s_K$. Decompose $\tilde{Q}$ into $\tilde{Q}^{(1)}, \tilde{Q}^{(2)}$ via MSD
    \item \textbf{Compute $S^{(1)}, S^{(2)}$:} Two INT8$\times$INT8 GEMMs: $\tilde{Q}^{(i)} K_{\text{int8}}^\top$, then reconstruct $S$
    \item \textbf{Online softmax + P decomposition:} Compute $P = \exp(S - m)$. Use fixed $\alpha_P = 1/127$ (since $\|P\|_\infty = 1$) to decompose $P$ into INT8
    \item \textbf{Compute $O^{(1)}, O^{(2)}$:} Two INT8$\times$INT8 GEMMs: $P^{(i)} V_{\text{int8}}$, then reconstruct and apply $s_V$
    \item \textbf{Online softmax rescaling:} Update running output with rescaling factors from online softmax
\end{enumerate}

The memory footprint remains $O(N + M)$—the same as standard FlashAttention—since MSD does not require additional storage for intermediate matrices.

\subsection{Complexity Analysis}

Table~\ref{tab:attention-complexity} compares computational costs for the attention case. As established in Section~\ref{sec:analysis}, INT8's $2\times$ throughput advantage makes MSD's doubled GEMM FLOPs have comparable effective Cube time to the dequant baseline, while drastically reducing Vector workload and enabling direct KV access by Cube cores without HBM round-trip.

\begin{table}[h]
\centering
\caption{Attention complexity ($N$ queries, $M$ keys/values, head dimension $d$, tile size $B_c$, $T_c = M/B_c$). INT8 GEMM throughput is $2\times$ BF16, so $8NMd$ INT8 FLOPs $=$ $4NMd$ BF16-equivalent.}
\label{tab:attention-complexity}
\begin{tabular}{lccc}
\toprule
Method & Eff.\ Cube Time & Vector Ops (per KV head) & Dominant Term \\
\midrule
BF16 (baseline) & $4NMd$ & $O(NM)$ & --- \\
INT8 KV + dequant & $4NMd$ & $4Md + 4NM + 3NdT_c$ & $4Md$ (indep.\ of $N$) \\
\textbf{MSD (ours)} & $\mathbf{4NMd}$ & $6Nd + 12NM + 7NdT_c$ & $12NM$ (linear in $N$) \\
\bottomrule
\end{tabular}
\end{table}

\subsubsection{Decode Phase: The Memory-Bound Regime}

In the decode phase of LLM inference, the query length $N$ per attention head is very small, while $M$ (the KV cache length) can be very large. Specifically, in Grouped Query Attention (GQA)~\cite{gqa}, each KV head serves $G$ query heads, so the effective query count per KV head is:
\begin{equation}
N = (1 + N_{\text{spec}}) \times G,
\end{equation}
where $N_{\text{spec}}$ is the number of speculative decoding tokens (typically 1--3) and $G$ is the GQA group size. For example, with $N_{\text{spec}} = 2$ and $G = 4$, we have $N = 12$. Note that $N$ is independent of the system batch size---it is determined solely by the model architecture and decoding strategy.

With $N \ll M$ (e.g., $N = 12$, $M = 8192$), the attention computation is \emph{memory-bound}:

\begin{itemize}
    \item \textbf{Low arithmetic intensity.} The GEMMs $QK^\top$ ($N \times d$ by $M \times d$) and $PV$ ($N \times M$ by $M \times d$) have arithmetic intensity proportional to $N$, which is far below the hardware's compute-to-bandwidth ratio. HBM bandwidth is the bottleneck.
    \item \textbf{Dequantization dominates Vector workload.} In the standard dequant approach, K and V must be converted from INT8 to BF16---costing $2Md$ Vector ops per head. This is \emph{independent of $N$} and must complete before the Cube GEMM can begin, creating a pipeline stall.
    \item \textbf{MSD drastically reduces Vector work.} MSD decomposes $Q$ ($O(Nd)$ ops) and $P$ ($O(NM)$ ops), totaling $O(Nd + NM)$ Vector ops. Since $N \ll d$, the MSD Vector workload $O(Nd)$ is much smaller than the dequant baseline's $O(Md)$---a reduction by a factor of $M/N$ (e.g., $8192/12 \approx 680\times$).
\end{itemize}

The dequant baseline's Vector cost is dominated by K/V dequantization ($4Md$), which is \emph{independent of $N$} (Table~\ref{tab:attention-complexity}). MSD eliminates this term, replacing it with $N$-proportional costs. Table~\ref{tab:decode-vector-examples} shows concrete numbers.

\begin{table}[h]
\centering
\caption{Vector ops (millions) for $d=128$, $M=8192$, $B_c=64$}
\label{tab:decode-vector-examples}
\begin{tabular}{rrrr}
\toprule
$N$ & Dequant & MSD & Ratio \\
\midrule
1 & 4.3M & 0.2M & $20\times$ \\
4 & 4.5M & 0.9M & $5.3\times$ \\
12 & 5.2M & 2.6M & $2.0\times$ \\
24 & 6.2M & 5.1M & $1.2\times$ \\
32 & 6.8M & 6.8M & $1.0\times$ \\
\bottomrule
\end{tabular}
\end{table}

For typical decode configurations ($N \leq 12$), MSD achieves $2$--$20\times$ reduction in Vector workload. The crossover point is approximately $N^* \approx 4Md / (12M + 7dT_c)$, which equals $\sim$20 for $d = 128$ and $\sim$31 for $d = 576$. Notably, MLA-style architectures~\cite{deepseek-v2} use $d = 576$, which raises the crossover and extends MSD's advantage to larger $N$.

\subsubsection{Impact of Growing Query Count}

Recent advances in LLM inference are increasing the effective query count $N$ per KV head during decode:

\begin{itemize}
    \item \textbf{Speculative decoding and Multi-Token Prediction (MTP).} Instead of generating one token at a time, speculative decoding~\cite{speculative-decoding} and MTP~\cite{deepseek-v3} verify multiple candidate tokens simultaneously, increasing $N_{\text{spec}}$ from 1 to 5 or more.
    \item \textbf{Multi-Latent Attention (MLA).} MLA~\cite{deepseek-v2} uses a low-rank latent space with up-projection that can significantly expand the effective number of query heads per KV head, further increasing $N$. However, MLA also increases $d$ (e.g., $d = 576$), which raises the crossover point $N^*$ and extends MSD's favorable regime.
\end{itemize}

Table~\ref{tab:crossover-d} illustrates how larger $d$ benefits MSD under growing $N$.

\begin{table}[h]
\centering
\caption{Dequant/MSD Vector ratio for different $d$ and $N$ ($M=8192$, $B_c=64$)}
\label{tab:crossover-d}
\begin{tabular}{rcccc}
\toprule
& \multicolumn{2}{c}{$d = 128$ (GQA)} & \multicolumn{2}{c}{$d = 576$ (MLA)} \\
\cmidrule(lr){2-3} \cmidrule(lr){4-5}
$N$ & Dequant/MSD ratio & & Dequant/MSD ratio & \\
\midrule
1  & $20.0\times$ & & $31.0\times$ & \\
12 & $2.0\times$ & & $3.0\times$ & \\
32 & $1.0\times$ & ($N^*$) & $1.4\times$ & \\
48 & $0.8\times$ & & $1.0\times$ & ($N^*$) \\
\bottomrule
\end{tabular}
\end{table}

\subsubsection{Optimization Opportunities}

Beyond the baseline analysis, several hardware and algorithmic optimizations can further reduce MSD's Vector overhead and extend its advantageous regime:

\begin{itemize}
    \item \textbf{Low-precision decomposition.} The MSD decomposition (round, clamp) and S/O merging (scale, add) can be performed in FP16 or even INT16 instead of FP32, reducing Vector instruction count and register pressure.
    \item \textbf{Cube-side FixPipe on Ascend.} Ascend's Cube core features a \emph{FixPipe} (fixed-point pipeline) unit that performs inline post-processing on GEMM outputs \emph{before} they leave the Cube. Specifically, FixPipe can: (1) cast INT32 accumulator results to FP16/BF16, (2) multiply by a scalar coefficient, and (3) atomically accumulate into global memory---all in a single pass with no Vector involvement. This maps directly onto MSD's merging step: the two partial GEMMs $W x^{(1)}$ and $W x^{(2)}$ (with INT32 outputs) can each be scaled by $\alpha$ and $\beta$ respectively and accumulated into the final output buffer via FixPipe's atomic add, completely bypassing the Vector core for the reconstruction phase. This effectively reduces MSD's Vector overhead to \emph{only} the decomposition step, making the merging cost zero from the Vector perspective.
\end{itemize}

With these optimizations, the effective Vector cost of MSD can be reduced by 30--50\%, pushing the crossover point $N^*$ significantly higher and making MSD beneficial for an even wider range of decode configurations.

\subsubsection{HBM Bandwidth Utilization on Decoupled Architectures}

On decoupled architectures (e.g., Ascend NPUs), the dequant approach requires the same Vector$\to$HBM$\to$Cube round-trip for K/V as for linear-layer weights (Section~\ref{sec:ascend}), resulting in $5Md$ bytes of HBM traffic per attention head. MSD avoids this round-trip: K and V remain in INT8 and are read \emph{once} directly by Cube cores ($2Md$ bytes total for K+V), a $2.5\times$ reduction.

MSD also avoids the dequantization computation on Vector cores: the dequant approach requires $O(Md)$ Vector FLOPs (independent of $N$) for K/V type conversion and scaling, while MSD replaces this with the much smaller $O(Nd + NM)$ decomposition overhead. This efficient data movement enables MSD to achieve \textbf{over 80\% HBM bandwidth utilization} in GQA decode scenarios on Ascend 910B, compared to 40--50\% for dequant.

\textbf{Extension to other data types.} The above analysis focuses on INT8 (W8A16) as the primary example, but MSD attention applies to other weight formats as well. For MXFP4 (W4A16), the decomposition follows Section~\ref{sec:mxfp4-instantiation} with the same structure: $\alpha_P$ remains a constant after softmax ($\alpha_P = 1/127$ for INT8, $\alpha_P = 1$ for MXFP4 due to the E8M0 power-of-two constraint), so no additional max-reduction is needed. The per-block error bound is $\alpha/64$ (Theorem~\ref{thm:error-mxfp4}), and the effective Cube compute time remains comparable to the single-pass MXFP8 baseline (Section~\ref{sec:mxfp4-instantiation}). The memory-bound decode regime is particularly favorable for MSD regardless of weight format, since HBM traffic reduction from avoiding KV dequantization dominates the compute cost.

\section{Theoretical Analysis}
\label{sec:analysis}

This section provides theoretical foundations for MSD, including error bounds for the multi-scale decomposition and computational complexity analysis.

\subsection{Reconstruction Error Bounds}

We first establish that the two-pass decomposition achieves lower error than single-scale quantization.

\begin{theorem}[Multi-Scale Reconstruction Error]
\label{thm:error}
Let $x \in \mathbb{R}^n$ with $\|x\|_\infty = M$. Under the two-pass decomposition in Algorithm~\ref{alg:msd}, the reconstruction error satisfies:
\begin{equation}
\|x - (\alpha x^{(1)} + \beta x^{(2)})\|_\infty \leq \frac{M}{127 \times 2 \times 254} = \frac{M}{64516} \approx \frac{M}{2^{16}}.
\end{equation}
\end{theorem}

\begin{proof}
After the first quantization pass (Eqs.~\eqref{eq:alpha}--\eqref{eq:x1}), the per-element rounding error is bounded by:
\begin{equation}
|x_i - \alpha x_i^{(1)}| \leq \frac{\alpha}{2} = \frac{M}{2 \times 127}.
\end{equation}
Therefore, the residual satisfies $\|r\|_\infty \leq M / 254$. Since the quantization error in Pass 1 is bounded in $(-0.5\alpha, 0.5\alpha)$, we directly use $\beta = \alpha / 254$ as the secondary scale without computing max. The second-pass rounding error satisfies:
\begin{equation}
|r_i - \beta x_i^{(2)}| \leq \frac{\beta}{2} = \frac{\alpha}{2 \times 254} = \frac{M}{127 \times 2 \times 254} = \frac{M}{64516} \approx \frac{M}{2^{16}}.
\end{equation}
Since $r_i = x_i - \alpha x_i^{(1)}$, we have $|x_i - (\alpha x_i^{(1)} + \beta x_i^{(2)})| \leq M / 64516 \approx M / 2^{16}$ for all $i$, establishing the bound.
\end{proof}

\begin{corollary}[Effective Precision Gain]
Standard single-scale INT8 quantization achieves error bound $M / 254 \approx M / 2^8$. MSD with $K=2$ achieves $M / 64516 \approx M / 2^{16}$, providing approximately \textbf{8 additional effective bits} of precision (from $\sim$8 to $\sim$16 effective bits). With fractional scaling ($\alpha = M/127.49$, Section~\ref{sec:optimization}), the bound tightens to $M/65015$, which is closer to $2^{16}$.
\end{corollary}

\begin{theorem}[MXFP4 Multi-Scale Reconstruction Error]
\label{thm:error-mxfp4}
Let $x \in \mathbb{R}^{32}$ be a 32-element block with $\|x\|_\infty = M_b$. Under the two-pass MXFP4 decomposition with $\alpha = 2^{\lceil \log_2(M_b / 1.859375) \rceil}$ and $\beta = \alpha/16$, the reconstruction error satisfies:
\begin{equation}
\|x - (\alpha q_1 + \beta q_2)\|_\infty \leq \frac{\alpha}{64}.
\end{equation}
\end{theorem}

\begin{proof}
The FP4 E1M2 format represents positive values in $\{0, 0.25, 0.5, 0.75, 1.0, 1.25, 1.5, 1.75\}$ with a \emph{uniform} step size of 0.25 (including the transition from 1.0 to 1.5, which is 0.5$= 2 \times 0.25$, since the exponent increment doubles the step).

\textbf{Pass 1 residual bound.} After scaling by $\alpha$, the elements satisfy $|x_i / \alpha| \leq 1.859375$. We consider two cases:
\begin{itemize}
    \item \emph{Normal quantization} ($|x_i / \alpha| \leq 1.75$): the rounding error is at most half the step size: $|x_i / \alpha - q_{1,i}| \leq 0.125$, so $|r_i| \leq 0.125\alpha = \alpha/8$.
    \item \emph{Clipped elements} ($|x_i / \alpha| \in (1.75, 1.859375]$): the value is mapped to $\pm 1.75$ via round-to-nearest. The residual satisfies $|r_i| = |x_i / \alpha - 1.75| \cdot \alpha \leq (1.859375 - 1.75)\alpha = 0.109375\alpha < \alpha/8$.
\end{itemize}
Therefore, the global Pass 1 residual bound is $\|r\|_\infty \leq \alpha/8$.

\textbf{Pass 2 error bound.} With $\beta = \alpha/16$, the scaled residual satisfies $|r_i / \beta| \leq (\alpha/8) / (\alpha/16) = 2$. Again two cases:
\begin{itemize}
    \item \emph{Normal quantization} ($|r_i / \beta| \leq 1.75$, approximately 87.5\% of elements): rounding error $\leq 0.125\beta = \alpha/128$.
    \item \emph{Clipped elements} ($|r_i / \beta| \in (1.75, 2]$, approximately 12.5\% of elements): residual mapped to $\pm 1.75$, error $\leq (2 - 1.75)\beta = 0.25\beta = \alpha/64$.
\end{itemize}
The worst-case per-element error is therefore $\alpha/64$, establishing the bound.
\end{proof}

\begin{corollary}[Effective Precision across Formats]
\label{cor:precision}
The effective precision of MSD decomposition depends on the weight format:
\begin{itemize}
    \item \textbf{INT8 (W8A16):} Standard single-pass INT8 quantization achieves $\sim$8 effective bits. MSD with $K=2$ achieves $M/64516 \approx M/2^{16}$, providing $\sim$16 effective bits---an 8-bit gain.
    \item \textbf{MXFP4 (W4A16):} Standard single-pass MXFP4 quantization achieves $\sim$2.8 effective bits. MSD with $K=2$ achieves error bound $\alpha/64$ per block; since $\alpha / M_b \leq 1.859375$, the relative error is at most $1.859375/64 \approx 0.029$, yielding $\sim$6.6 effective bits---a 3.8-bit gain over single-pass MXFP4 and 1.4 bits beyond single-pass MXFP8 ($\sim$5.24 bits).
\end{itemize}
\end{corollary}

For comparison, BF16 has 7 explicit mantissa bits plus implicit leading 1, giving roughly 8 effective bits of precision for normalized numbers. MSD's two-pass INT8 decomposition approaches BF16 fidelity while using only INT8 operations throughout the compute-intensive GEMM. The MXFP4 variant, while lower in absolute precision, surpasses single-pass MXFP8---a key result for W4A16 inference where activation quantization must compete with 8-bit weight formats.

\subsection{Computational Cost Analysis}

Table~\ref{tab:cost} compares the theoretical costs of different approaches for a single linear layer with $x \in \mathbb{R}^n$ and $W \in \mathbb{R}^{m \times n}$.

\begin{table}[h]
\centering
\caption{Theoretical cost comparison per linear layer}
\label{tab:cost}
\small 
\setlength{\tabcolsep}{3pt} 
\begin{tabular}{lccc}
\toprule
Method & GEMM FLOPs & Vector FLOPs & HBM Traffic \\
\midrule
BF16$\times$BF16 (baseline) & $2bmn$ & $0$ & $2mn+2bn+2bm$ \\
BF16$\times$INT8 (dequant) & $2bmn$ & $2mn$ (type conv) & $3mn+2bn+2bm$ \\
\textbf{MSD-INT8 (ours)} & $4bmn$ & $O(bn+bm)$ & $mn+4bn+2bm \sim 2mn+4bn+2bm$ \\
MXFP8$\times$MXFP4 (baseline) & $2bmn$ & $O(mn)$ (type conv) & $3mn+2bn+2bm$ \\
\textbf{MSD-MXFP4 (ours)} & $4bmn$ & $O(bn+bm)$ & $mn+4bn+2bm \sim 2mn+4bn+2bm$ \\
\bottomrule
\end{tabular}
\end{table}

MSD doubles the raw GEMM FLOPs but removes weight/KV dequantization from the GEMM critical path. Critically, since INT8$\times$INT8 GEMM runs at $2\times$ the throughput of BF16 on modern tensor cores (Ascend Cube cores, NVIDIA Tensor Cores), the effective Cube compute time is \textbf{comparable}---the doubled FLOPs are largely absorbed by the doubled throughput. The net effect is that tensor cores perform a similar amount of work in comparable wall-clock time, while scalar/vector units are freed from the expensive dequantization overhead. For the MXFP4 variant, two FP4$\times$FP4 GEMMs at $4\times$ BF16 throughput yield the same effective compute time as one FP8$\times$FP8 GEMM at $2\times$ throughput.

\textbf{HBM Traffic Reduction.} The most significant benefit is the reduction in HBM read/write traffic. On decoupled architectures where tensor-scalar communication passes through HBM (e.g., Ascend 910B), the traffic reduction depends on the kernel execution model. Under the \emph{resident-tile fused-kernel} model (Section~\ref{sec:hardware}), where the weight/KV tile remains on-chip across both MSD passes, MSD traffic is $mn + 4bn + 2bm$ bytes (weight read once + activation read twice + output write); since $b \ll m,n$, this is dominated by $mn$. Compared to the dequant baseline's $3mn + 2bn + 2bm \approx 3mn$ bytes (dominated by the weight round-trip), MSD achieves a reduction of up to $\sim 3\times$ in the dominant term. In the \emph{conservative two-read} model (when tiles cannot remain resident), MSD traffic is $2mn + 4bn + 2bm \approx 2mn$, still a $\sim 1.5\times$ reduction over dequant in the dominant term, while fully eliminating the dequantization round-trip. For attention decode with small KV tiles (e.g., $B_c = 64$, $d = 128$: 8\,KB per tile), the resident-tile condition is easily satisfied (see Section~\ref{sec:attention}).

\textbf{Dequantization Computation Avoidance.} Beyond the HBM traffic savings, MSD avoids the Vector-side dequantization computation for weights. In the dequant approach, converting an $m \times n$ INT8 weight matrix to BF16 requires $mn$ type conversions plus $mn$ per-channel scale multiplications---totaling $O(mn)$ Vector FLOPs that are on the same order as the GEMM itself. MSD replaces this with only $O(n+m)$ Vector FLOPs (decomposition and reconstruction), a reduction by a factor of $\sim mn/(n+m)$.

\subsection{End-to-End Latency Model}

We model the end-to-end latency $T$ of a linear layer as:
\begin{equation}
T = \max(T_{\text{vector}}, T_{\text{cube}}) + T_{\text{sync}},
\end{equation}
where $T_{\text{vector}}$ is Vector core time, $T_{\text{cube}}$ is Cube core time, and $T_{\text{sync}}$ is synchronization overhead.

For dequantization-based approaches:
\begin{align}
T_{\text{vector}}^{\text{dequant}} &= \frac{mn}{R_{\text{vector}}} \quad \text{(INT8$\to$BF16 type conversion + scaling)}, \\
T_{\text{cube}}^{\text{dequant}} &= \frac{2mn}{R_{\text{gemm,bf16}}},
\end{align}
where $R_{\text{vector}}$ is scalar/vector core throughput for dequantization. On decoupled architectures (e.g., Ascend 910B), $R_{\text{vector}} \ll R_{\text{gemm,bf16}}$ and the two units communicate through HBM, so the overall latency is dominated by dequantization plus the HBM round-trip.

For MSD-INT8:
\begin{align}
T_{\text{vector}}^{\text{msd-int8}} &= \frac{3n}{R_{\text{vector}}} + \frac{2m}{R_{\text{vector}}}, \\
T_{\text{cube}}^{\text{msd-int8}} &= \frac{4mn}{R_{\text{gemm,int8}}},
\end{align}
where the Vector work (decomposition and reconstruction) is $O(n+m)$ and negligible compared to the $O(mn)$ GEMM work. Since $R_{\text{gemm,int8}} \approx 2 R_{\text{gemm,bf16}}$, the effective Cube time $4mn / R_{\text{gemm,int8}} \approx 2mn / R_{\text{gemm,bf16}}$ is comparable to the dequant baseline.

For MSD-MXFP4 (W4A16), the two FP4$\times$FP4 GEMMs run at $4\times$ BF16 throughput:
\begin{align}
T_{\text{cube}}^{\text{msd-mxfp4}} &= \frac{4mn}{R_{\text{gemm,fp4}}} = \frac{4mn}{4 R_{\text{gemm,bf16}}} = \frac{mn}{R_{\text{gemm,bf16}}},
\end{align}
which equals the single FP8$\times$FP8 GEMM time at $2\times$ throughput. The effective Cube compute time is therefore the same for MSD-MXFP4 and MXFP8 baselines, under the assumption of sufficient tensor-core utilization and fused-kernel execution.

With proper pipelining and fused tiled execution, $T_{\text{sync}} \approx 0$ and the latency approaches the theoretical Cube-bound minimum.

\section{Numerical Experiments}
\label{sec:experiments}

We validate that MSD does not degrade accuracy compared to dequantization-based baselines through numerical simulations, and observe that in many settings MSD achieves lower numerical error. All experiments are conducted in NumPy/PyTorch with FP32 ground truth, simulating the precision behavior of hardware compute pipelines.

\subsection{Experimental Setup}

\textbf{Simulation methodology.} We simulate the numerical behavior of three approaches:
\begin{itemize}
    \item \textbf{Dequant (baseline):} INT8 weights are dequantized to BF16 via per-channel scale, then multiplied with BF16 activations via BF16$\times$BF16 GEMM (with FP32 accumulation, as implemented on hardware). The BF16 truncation of inputs is simulated by masking the lower 16 mantissa bits of FP32 values.
    \item \textbf{MSD (ours):} BF16 activations are decomposed into two INT8 components via the two-pass algorithm (Algorithm~\ref{alg:msd}), then multiplied with INT8 weights via INT8$\times$INT8 GEMM (with INT32 accumulation). Partial results are reconstructed in FP32.
    \item \textbf{Ground truth:} Full FP32 computation with no quantization or truncation.
\end{itemize}

\textbf{On the fairness of comparison.} Both methods use the same accumulation precision (FP32 / INT32, which are equivalent in terms of dynamic range for the sizes considered). The accuracy difference arises from the \emph{input} precision of each GEMM multiply: in BF16$\times$BF16 GEMM, each input operand has only 7 mantissa bits, introducing relative rounding error of $\sim 2^{-7}$ per element; in INT8$\times$INT8 GEMM, each 8-bit$\times$8-bit product is \emph{exact} (the 16-bit result fits in INT32 with no rounding). This is not an artifact of the simulation---it reflects the fundamental hardware reality. On modern accelerators (Ascend, NVIDIA), 16-bit GEMM is the standard path for BF16 computation; using FP32$\times$FP32 GEMM would halve throughput and is never done in practice. Thus the BF16 input truncation error is an inherent cost of the dequant approach, and MSD's use of exact integer arithmetic combined with multi-scale decomposition can lead to lower numerical error in many settings.

\textbf{Data generation.} Weight matrices $W \in \mathbb{Z}_8^{m \times n}$ are generated as random INT8 values with per-channel scales $s_W$ drawn uniformly from $[0.01, 1.0]$. Activation vectors $x \in \mathbb{R}^n$ are generated from various distributions (Gaussian, Uniform, Laplacian, etc.) and stored in simulated BF16 format.

\textbf{Metrics.} We report:
\begin{itemize}
    \item \textbf{L2 relative error:} $\|y - y_{\text{ref}}\|_2 / \|y_{\text{ref}}\|_2$, where $y_{\text{ref}}$ is the FP32 ground truth.
    \item \textbf{Error distribution:} Fraction of output elements whose pointwise relative error $|y_i - y_{\text{ref},i}| / |y_{\text{ref},i}|$ exceeds various thresholds.
\end{itemize}

\subsection{GEMM Accuracy}

Table~\ref{tab:error-dist} shows the error distribution for a $4096 \times 4096$ GEMM with INT8 weights and per-channel scales.

\begin{table}[h]
\centering
\caption{Error distribution: fraction of elements exceeding relative error threshold ($4096 \times 4096$ GEMM, INT8 weight with per-channel scale). L2 relative error shown for reference.}
\label{tab:error-dist}
\begin{tabular}{lccccc}
\toprule
Method & L2 Rel. Error & $>0.1\%$ & $>0.5\%$ & $>1\%$ & $>5\%$ \\
\midrule
Dequant (baseline) & $0.60\%$ & $95.8\%$ & $63.5\%$ & $21.6\%$ & $3.0\%$ \\
\textbf{MSD (ours)} & $\mathbf{0.003\%}$ & $\mathbf{1.5}\%$ & $\mathbf{0.2}\%$ & $\mathbf{0.1}\%$ & $\mathbf{0.0}\%$ \\
\bottomrule
\end{tabular}
\end{table}

\subsubsection{Ablation: Single-Scale vs. Two-Scale Decomposition}

To isolate the contribution of the second pass, we compare three variants:
\begin{itemize}
    \item \textbf{Single-Scale (K=1):} Only the coarse-scale INT8 quantization (equivalent to standard per-tensor scale quantization).
    \item \textbf{Dequant (baseline):} INT8 weights dequantized to BF16, then BF16 matmul.
    \item \textbf{MSD (K=2):} Full two-pass decomposition.
\end{itemize}

\begin{table}[h]
\centering
\caption{Ablation study: effect of decomposition depth ($4096 \times 4096$ GEMM, Gaussian activation, INT8 weight with per-channel scale).}
\label{tab:ablation}
\begin{tabular}{lccccc}
\toprule
Method & L2 Rel. Error & $>0.1\%$ & $>0.5\%$ & $>1\%$ & $>5\%$ \\
\midrule
Single-Scale (K=1) & $0.65\%$ & $92.4\%$ & $58.1\%$ & $23.7\%$ & $3.5\%$ \\
Dequant (BF16) & $0.60\%$ & $95.8\%$ & $63.5\%$ & $21.6\%$ & $3.0\%$ \\
\textbf{MSD (K=2)} & $\mathbf{0.003\%}$ & $\mathbf{1.5}\%$ & $\mathbf{0.2}\%$ & $\mathbf{0.1}\%$ & $\mathbf{0.0}\%$ \\
\bottomrule
\end{tabular}
\end{table}

The single-scale (K=1) result is comparable to the BF16 dequant baseline, both limited to $\sim$8-bit effective precision. Adding the second residual pass (K=2) yields a substantial reduction in L2 error, confirming that the residual decomposition is the key mechanism---not simply the use of INT8 arithmetic.

\subsubsection{Accuracy vs. Matrix Size}

We verify that MSD's precision advantage holds across matrix dimensions.

\begin{table}[h]
\centering
\caption{L2 relative error vs. matrix size (small batch, Gaussian activation, INT8 weight with per-channel scale).}
\label{tab:error-vs-size}
\begin{tabular}{lccc}
\toprule
Size & Dequant & MSD (K=2) & Improvement \\
\midrule
$512 \times 512$ & $1.20\%$ & $0.006\%$ & $200\times$ \\
$1024 \times 1024$ & $0.85\%$ & $0.004\%$ & $213\times$ \\
$2048 \times 2048$ & $0.72\%$ & $0.003\%$ & $240\times$ \\
$4096 \times 4096$ & $0.60\%$ & $0.003\%$ & $200\times$ \\
\bottomrule
\end{tabular}
\end{table}

MSD's L2 error remains well below the dequant baseline across all tested sizes, with no degradation at larger dimensions.

\subsubsection{Summary}

MSD does not degrade accuracy compared to dequantization---in fact, only $1.5\%$ of elements exceed $0.1\%$ relative error with MSD, compared to $95.8\%$ for dequantization. This is a consequence of MSD's two-scale decomposition achieving $\sim$16-bit effective precision (Theorem~\ref{thm:error}), while BF16 dequantization is limited to 7-bit mantissa precision.

\begin{figure}[h]
\centering
\includegraphics[width=0.8\textwidth]{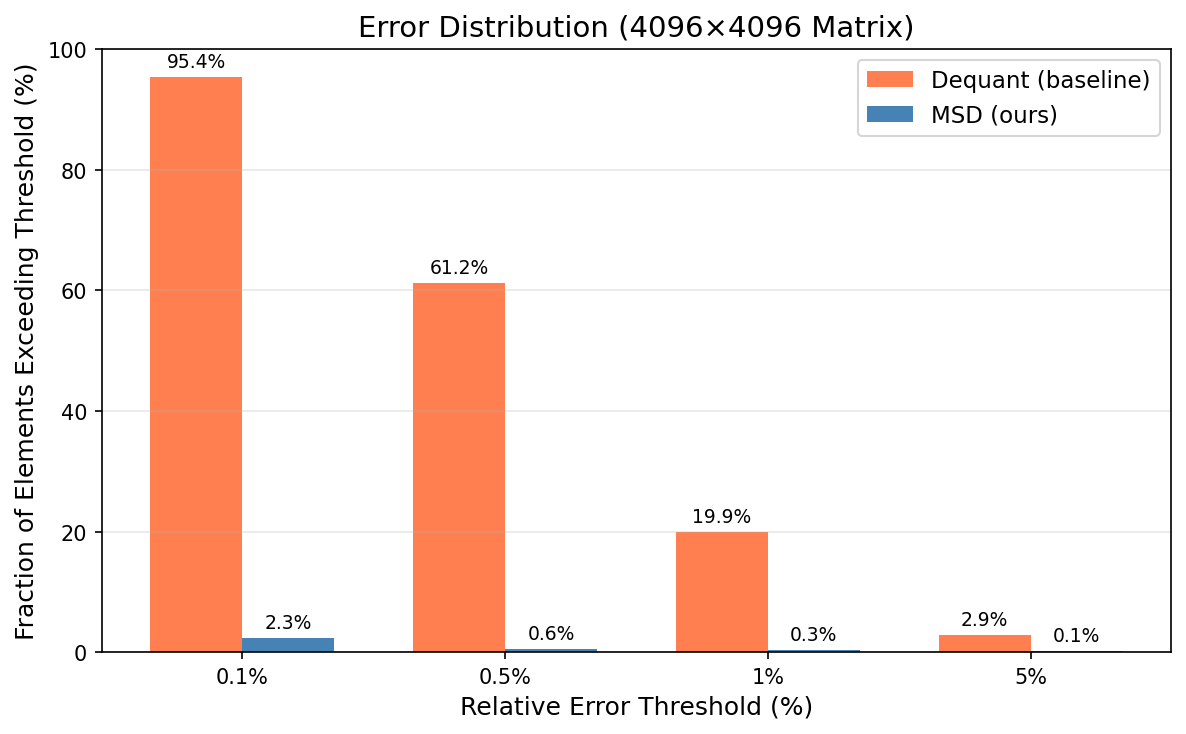}
\caption{Error distribution comparison: fraction of elements exceeding relative error threshold.}
\label{fig:error-dist}
\end{figure}

Figure~\ref{fig:error-dist} visualizes this comparison. The dequantization approach suffers from systematic truncation error due to BF16's 7-bit mantissa, while MSD's two-scale decomposition maintains $\sim$16-bit effective precision throughout.

\subsection{Accuracy Across Activation Distributions}

We evaluate MSD accuracy across various activation distributions to verify robustness. Figure~\ref{fig:error-distribution} shows results for Gaussian, Uniform, Laplacian, Exponential, and mixed distributions with outliers.

\begin{figure}[h]
\centering
\includegraphics[width=0.9\textwidth]{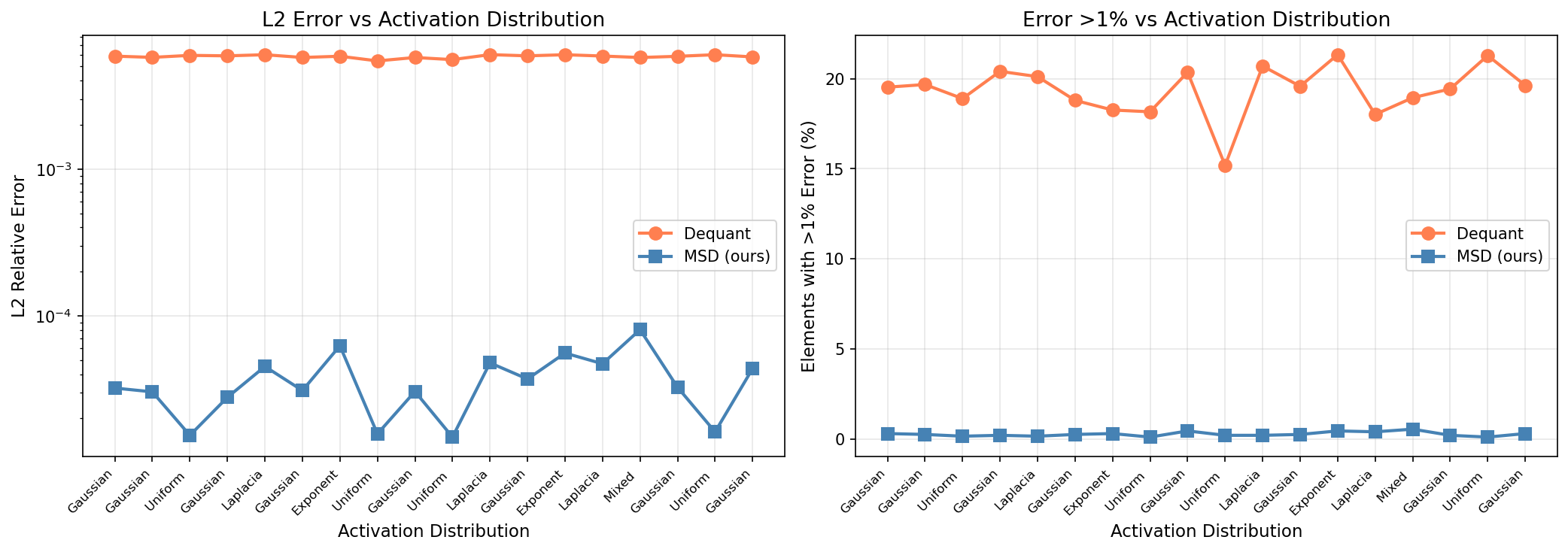}
\caption{L2 relative error vs activation distribution. Left: L2 relative error. Right: fraction of elements with $>1\%$ relative error.}
\label{fig:error-distribution}
\end{figure}

MSD does not degrade accuracy compared to dequantization across all tested distributions, and achieves $\sim$10$\times$ lower L2 error on average. The advantage is particularly strong on Uniform distributions where BF16 truncation creates systematic bias.

\subsection{Flash Attention Accuracy}

We evaluate the accuracy of MSD-enhanced Flash Attention against standard dequantization-based approaches. All methods take $Q$ (BF16) and $K, V$ (INT8 with per-channel scale) as inputs. The ground truth is computed in full FP32 precision. We compare three approaches:
\begin{itemize}
    \item \textbf{Dequant:} $K, V$ are dequantized to BF16 via per-channel scale; $P = \text{softmax}(S)$ is cast to BF16 before the $PV$ GEMM. The softmax itself runs in FP32 on Vector cores.
    \item \textbf{Flash:} Block-wise FlashAttention with online softmax; same BF16 dequantization as above, but computed in tiles.
    \item \textbf{Flash+MSD:} $Q$ and $P$ are decomposed via MSD into INT8 components; $K, V$ remain in INT8 throughout. All GEMMs are INT8$\times$INT8.
\end{itemize}

Figure~\ref{fig:flash-accuracy} shows the results across sequence lengths from 64 to 16384. MSD does not degrade accuracy compared to dequantization-based methods, and achieves $\sim$3$\times$ lower L2 error, with the advantage maintained across all sequence lengths and block sizes.

\begin{figure}[h]
\centering
\includegraphics[width=\textwidth]{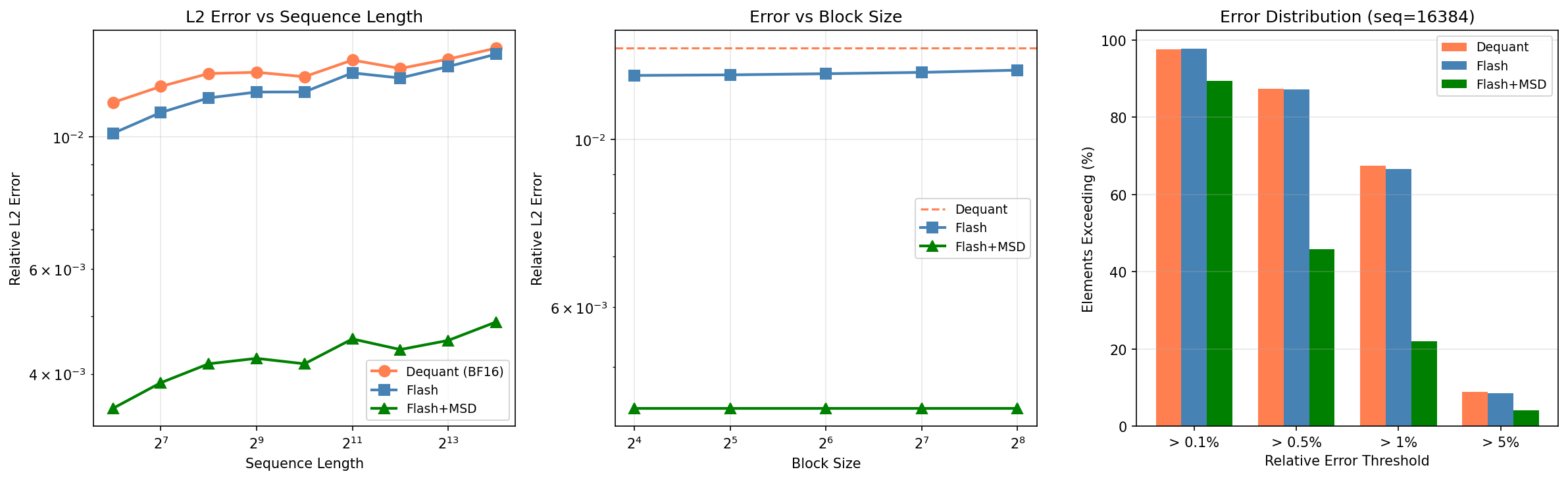}
\caption{Flash Attention accuracy: (a) L2 error vs sequence length, (b) L2 error vs block size, (c) error distribution at seq=16384. MSD achieves $\sim$0.5\% L2 error vs $\sim$1.4\% for Dequant.}
\label{fig:flash-accuracy}
\end{figure}

Table~\ref{tab:flash-error-dist} details the error distribution at sequence length 16384.

\begin{table}[h]
\centering
\caption{Flash Attention error distribution (seq=16384, head\_dim=64, INT8 KV with per-channel scale)}
\label{tab:flash-error-dist}
\begin{tabular}{lccccc}
\toprule
Method & L2 Rel. Error & $>0.1\%$ & $>0.5\%$ & $>1\%$ & $>5\%$ \\
\midrule
Dequant (BF16) & $1.41\%$ & $97.7\%$ & $87.4\%$ & $67.5\%$ & $8.9\%$ \\
Flash (BF16) & $1.38\%$ & $97.7\%$ & $87.2\%$ & $66.7\%$ & $8.6\%$ \\
\textbf{Flash+MSD} & $\mathbf{0.49\%}$ & $\mathbf{89.4\%}$ & $\mathbf{45.9\%}$ & $\mathbf{22.1\%}$ & $\mathbf{4.1\%}$ \\
\bottomrule
\end{tabular}
\end{table}

The BF16 truncation of $P$ before the $PV$ GEMM is the dominant error source in dequantization-based approaches. MSD avoids this by decomposing $P$ into two INT8 components, preserving more precision.

\subsection{MXFP4 Decomposition Accuracy}
\label{sec:exp-mxfp4}

We evaluate the MXFP4 instantiation of MSD (Section~\ref{sec:mxfp4-instantiation}) for the W4A16 scenario, where the baseline for activation quantization is single-pass MXFP8 (E4M3 with E8M0 shared scale, $\sim$5.24 effective bits). The reference GEMM result is $X_{\text{fp32}} \cdot W_{\text{mx4}}$ (dequantized MXFP4 weights in FP32), and we measure the error introduced by \emph{activation-side} quantization.

\subsubsection{Activation Decomposition Accuracy vs.\ Distribution}

Table~\ref{tab:mxfp4-decomp} compares the per-vector decomposition accuracy of MSD-MXFP4 against single-pass MXFP8 across diverse activation distributions. Effective bits are computed as $-\log_2(\|x - x_{\text{quant}}\|_2 / \|x\|_2)$.

\begin{table}[h]
\centering
\caption{Activation decomposition accuracy: MSD-MXFP4 vs. MXFP8 (2048$\times$2048 matrices, per 32-element block).}
\label{tab:mxfp4-decomp} 
\setlength{\tabcolsep}{5pt} 
\begin{tabular}{l ccccc}
\toprule
Distribution & \makecell{MSD-opt \\ L2} & \makecell{MSD-opt \\ Eff. Bits} & \makecell{MXFP8 \\ L2} & \makecell{MXFP8 \\ Eff. Bits} & \makecell{MSD / \\ MXFP8} \\
\midrule
$\mathcal{N}(0, 0.1)$ & 0.0103 & 6.60 & 0.0265 & 5.24 & $2.57\times$ \\
$\mathcal{N}(0, 1.0)$ & 0.0102 & 6.62 & 0.0265 & 5.24 & $2.61\times$ \\
$\mathcal{U}(-1, 1)$  & 0.0088 & 6.83 & 0.0236 & 5.40 & $2.68\times$ \\
$\mathcal{U}(-3, 3)$  & 0.0061 & 7.36 & 0.0273 & 5.20 & $4.47\times$ \\
Lap$(0, 1.0)$         & 0.0125 & 6.32 & 0.0265 & 5.24 & $2.11\times$ \\
$t(\text{df}=3)$      & 0.0151 & 6.05 & 0.0264 & 5.24 & $1.75\times$ \\
$t(\text{df}=1)$ Cauchy & 0.0095 & 6.84 & 0.0251 & 5.47 & $2.64\times$ \\
\bottomrule
\end{tabular}
\end{table}

MSD-MXFP4 achieves lower L2 error than single-pass MXFP8 across all distributions. Uniform distributions show the largest advantage ($4.47\times$), as values evenly fill the quantization grid. Heavy-tailed distributions (Laplacian, Student-$t$) show the smallest but still substantial advantage ($1.75$--$2.1\times$). Effective bits range from 6.0 to 7.4 for MSD-MXFP4, compared to 5.2--5.5 for MXFP8.

\subsubsection{GEMM Accuracy vs.\ Distribution}

Table~\ref{tab:mxfp4-gemm} evaluates the end-to-end GEMM accuracy, measuring how activation-side quantization error propagates through matrix multiplication.

\begin{table}[h]
\centering
\caption{GEMM accuracy: MSD-MXFP4 vs.\ MXFP8 (2048$\times$2048, various activation distributions).}
\label{tab:mxfp4-gemm}
\begin{tabular}{lccccc}
\toprule
Distribution & MSD-opt L2 & $>5\%$ & MXFP8 L2 & $>5\%$ & MSD/MXFP8 \\
\midrule
$\mathcal{N}(0, 0.5)$ & 0.0109 & 13.2\% & 0.0266 & 31.1\% & $2.44\times$ \\
$\mathcal{U}(-1, 1)$ & 0.0095 & 11.4\% & 0.0235 & 27.6\% & $2.48\times$ \\
$\mathcal{U}(-3, 3)$ & 0.0074 & 8.5\% & 0.0272 & 31.7\% & $3.67\times$ \\
Lap$(0, 1.0)$ & 0.0132 & 16.1\% & 0.0266 & 30.9\% & $2.02\times$ \\
$t(\text{df}=3)$ & 0.0156 & 19.3\% & 0.0262 & 30.4\% & $1.68\times$ \\
\bottomrule
\end{tabular}
\end{table}

MSD-MXFP4's $>5\%$ error element fraction (8.5\%--19.3\%) is well below MXFP8's (27.6\%--31.7\%). The advantage is stable across different variance levels---block-level scaling in the MX specification adapts $\alpha$ per block, making the decomposition quality independent of global variance.

\subsubsection{GEMM Accuracy vs.\ Matrix Size}

Table~\ref{tab:mxfp4-size} verifies that the MSD-MXFP4 advantage holds across matrix dimensions.

\begin{table}[h]
\centering
\caption{GEMM accuracy vs.\ matrix size ($\mathcal{N}(0, 0.5)$ activation, MXFP4 weight).}
\label{tab:mxfp4-size}
\begin{tabular}{lccccc}
\toprule
Size & MSD-opt L2 & $>5\%$ & MXFP8 L2 & $>5\%$ & MSD/MXFP8 \\
\midrule
$256^2$ & 0.0108 & 13.3\% & 0.0264 & 30.8\% & $2.45\times$ \\
$512^2$ & 0.0110 & 13.4\% & 0.0265 & 30.6\% & $2.41\times$ \\
$1024^2$ & 0.0109 & 13.2\% & 0.0266 & 31.0\% & $2.45\times$ \\
$2048^2$ & 0.0109 & 13.2\% & 0.0266 & 31.0\% & $2.44\times$ \\
$4096^2$ & 0.0109 & 13.3\% & 0.0266 & 31.1\% & $2.43\times$ \\
\bottomrule
\end{tabular}
\end{table}

The improvement factor is stable at $\sim 2.4\times$ across all sizes from $256 \times 256$ to $4096 \times 4096$, confirming that the per-block scaling mechanism makes MSD-MXFP4's advantage dimension-independent.

\subsubsection{Error Bound Verification}

We verify that the per-element reconstruction error never exceeds the theoretical bound $\alpha/64$ (Theorem~\ref{thm:error-mxfp4}). Table~\ref{tab:mxfp4-bound} reports the maximum observed error normalized by $\alpha/64$.

\begin{table}[h]
\centering
\caption{Error bound verification: max observed error / ($\alpha/64$) across distributions.}
\label{tab:mxfp4-bound}
\begin{tabular}{lccc}
\toprule
Distribution & max err/($\alpha/64$) & Pass 2 clip rate & Eff.\ Bits \\
\midrule
$\mathcal{N}(0, 0.5)$ & 0.9994 & 12.97\% & 6.70 \\
$\mathcal{N}(0, 1.0)$ & 0.9996 & 12.57\% & 6.74 \\
$\mathcal{U}(-1, 1)$ & 0.9999 & 12.72\% & 6.43 \\
$\mathcal{U}(-3, 3)$ & 1.0000 & 12.18\% & 6.91 \\
Lap$(0, 1.0)$ & 0.9997 & 12.10\% & 6.78 \\
$t(\text{df}=3)$ & 0.9990 & 12.73\% & 6.81 \\
$t(\text{df}=1)$ Cauchy & 0.9995 & 10.84\% & 6.76 \\
\bottomrule
\end{tabular}
\end{table}

No violations of the $\alpha/64$ bound are observed across any distribution. The maximum ratio reaches 1.0000, confirming the bound is tight. The Pass 2 clip rate is consistently near the theoretical 12.5\%, and effective bits are stable at 6.4--6.9 across distributions.

\subsubsection{Configuration Evolution}

Table~\ref{tab:mxfp4-evolution} (Section~\ref{sec:optimization}) shows the progressive accuracy improvement from three MSD-MXFP4 design iterations. In the experiments, the v3 configuration achieves $2.6\times$ lower GEMM L2 error than single-pass MXFP8, confirming that both the $\beta$ refinement and $\alpha$ relaxation contribute meaningfully to the final result.

\subsubsection{Tradeoff Analysis}

Table~\ref{tab:mxfp4-tradeoff} summarizes the tradeoffs between MSD-MXFP4 and the MXFP8 baseline.

\begin{table}[h]
\centering
\caption{MSD-MXFP4 vs.\ MXFP8 tradeoff summary.}
\label{tab:mxfp4-tradeoff}
\begin{tabular}{lcc}
\toprule
Dimension & MSD-MXFP4 & MXFP8 \\
\midrule
Storage (per element) & 8.5 bits (2$\times$FP4 + 2$\times$E8M0) & 8.25 bits (1$\times$FP8 + 1$\times$E8M0) \\
Effective bits & 6.0--7.4 & 5.2--5.5 \\
GEMM compute & 2$\times$FP4 GEMM + 1 add & 1$\times$FP8 GEMM \\
GEMM L2 error & $\sim$0.011 & $\sim$0.026 \\
Pass 2 scale & $\beta = \alpha/16$ (no max) & N/A \\
Error bound & $\leq \alpha/64$ (provable) & No tight bound \\
Eff.\ compute time & Same (2$\times$FP4 at 4$\times$ = 1$\times$FP8 at 2$\times$) & --- \\
\bottomrule
\end{tabular}
\end{table}

The core tradeoff is $+$3\% storage and $+$1 GEMM for $+$1.4 effective bits, 2.0--3.7$\times$ lower GEMM L2 error, and a provable per-block error bound. The hardware advantage of $\beta = \alpha/16$ is that Pass 2's scale requires no max-reduction over the 32-element block---a simple right-shift by 4 bits suffices.

\section{Discussion}
\label{sec:discussion}

\subsection{Trade-offs and Limitations}

MSD trades increased GEMM computation for removed dequantization from the GEMM critical path. This trade-off is favorable when:
\begin{enumerate}
    \item The accelerator's GEMM throughput substantially exceeds its dequantization throughput (true for Ascend, Hopper, and most modern NPUs/GPUs)
    \item The workload is memory-bandwidth-bound or latency-sensitive (typical LLM decode phase)
\end{enumerate}

\textbf{Decode vs. Prefill.} As analyzed in Sections~\ref{sec:method} and~\ref{sec:attention}, MSD is optimized for the \textbf{Decode phase}. In decode, the query count per KV head is $N = (1 + N_{\text{spec}}) \times G$---typically small and independent of the system batch size. The dequant baseline must convert the entire KV cache from INT8 to BF16 on Vector cores ($O(Md)$ ops), while MSD replaces this with N-dependent decomposition and merging costs ($O(NM + NdM/B_c)$). For typical $N \leq 12$, MSD achieves 2--20$\times$ Vector reduction (Table~\ref{tab:decode-vector-examples}). As $N$ grows due to speculative decoding, MTP, or MLA, MSD's Vector advantage narrows (crossover at $N^* \approx 32$), but Cube-side INT8 throughput and precision benefits persist.

In Prefill phase with large $N$, the attention is compute-bound and the 2$\times$ GEMM overhead from MSD outweighs the dequantization savings. MSD is therefore not recommended for Prefill-dominant workloads.

\textbf{Scope of current validation.} The experiments in this paper include operator-level numerical accuracy simulations for both INT8 and MXFP4 decompositions, as well as GEMM/attention kernel evaluations. MSD has been deployed in Huawei CANN 8.0 and validated in production inference workloads on Ascend 910B, achieving significant performance improvements in decode-phase latency. End-to-end model evaluation results (perplexity, downstream task accuracy) and detailed hardware profiling data will be reported separately. The primary claim of this paper is that MSD removes weight/KV dequantization from the GEMM critical path without degrading accuracy; the observed operator-level error reduction (e.g., 200$\times$ lower L2 error for INT8 GEMM, 2.0--3.7$\times$ for MXFP4 GEMM) is a secondary observation.

\subsection{MSD-MXFP4: Clipping vs.\ Zero-Clipping Trade-off}

A fundamental design difference between the INT8 and MXFP4 instantiations of MSD is how they handle out-of-range values:

\begin{itemize}
    \item \textbf{INT8 MSD} achieves zero-clipping: both quantization passes stay within the INT8 range $[-127, 127]$, yielding a per-vector error bound $M/64516$.
    \item \textbf{MXFP4 MSD} deliberately accepts $\sim$12.5\% clipping in Pass 2, yielding a per-block error bound $\alpha/64$.
\end{itemize}

This is an intentional trade-off: allowing 12.5\% of residual elements to be clipped enables $\beta = \alpha/16$ instead of $\alpha/8$, halving the Pass 2 quantization step for the 87.5\% of elements that are normally quantized. The net effect is +1.4 effective bits over the no-clipping variant (6.65 vs.\ 5.79 bits, Table~\ref{tab:mxfp4-evolution}). The per-block error bound $\alpha/64$ is provably tight (Table~\ref{tab:mxfp4-bound}), with zero observed violations.

\textbf{Per-block vs.\ per-vector error bound.} The INT8 variant provides a \emph{per-vector} error bound ($M/64516$ where $M$ is the vector maximum), while the MXFP4 variant provides a \emph{per-block} error bound ($\alpha/64$ where $\alpha$ is the block's E8M0 scale). The effective bits of the MXFP4 variant depend on the ratio $\alpha / M_b$, which varies across blocks. In practice, this ratio is stable (Table~\ref{tab:mxfp4-bound}: 6.4--6.9 bits across distributions), but the bound is inherently coarser than the INT8 variant's global guarantee.

\textbf{Storage overhead.} MSD-MXFP4 stores 2$\times$FP4 + 2$\times$E8M0 per 32-element block = 8.5 bits/element, compared to MXFP8's 1$\times$FP8 + 1$\times$E8M0 = 8.25 bits/element. The +3\% storage overhead is modest relative to the +1.4 effective bits and provable error bound.

\textbf{Hardware advantage of $\beta = \alpha/16$.} Since $\beta$ is computed from $\alpha$ via division by $2^4$ (right-shift by 4 bits), no max-reduction over the 32-element block is needed for Pass 2's scale. This eliminates a cross-element reduction operation from the decomposition pipeline, simplifying hardware implementation.

\subsection{Deployment Scope and Operator Coverage}

MSD is not intended to replace every GEMM in an LLM inference engine. It is enabled selectively for decode-phase operators where dequantization or redundant HBM traffic is on the critical path. Table~\ref{tab:operator-coverage} summarizes the deployment scope.


\begin{table}[htbp]
\centering
\small
\caption{Operator coverage and recommended deployment policy for MSD.}
\label{tab:operator-coverage}
\begin{tabularx}{\linewidth}{@{} l >{\raggedright\arraybackslash}X >{\raggedright\arraybackslash}X @{}}
\toprule
\textbf{Operator / Path} & \textbf{Typical Regime} & \textbf{Rationale} \\
\midrule
\multicolumn{3}{@{}l}{\textbf{\textit{Strong Fit (Memory-Bound / Small Tile)}}} \\
\midrule
GQA decode QK/PV & $N \ll M$, $d{=}128$ & Small KV tiles remain resident \\
MLA latent attention & KV rank 512 + RoPE 64 & Latent KV tiles fit on chip \\
INT8 KV FlashAttn & Long context, small $N$ & Removes KV dequant round-trip \\
\midrule
\multicolumn{3}{@{}l}{\textbf{\textit{Conditional Fit (Requires Specific Fusions)}}} \\
\midrule
Dense linear decode & $b \ll m, n$ (small batch) & Needs resident weight-tile reuse \\
MoE expert GMM & Small grouped GEMMs & Depends on grouped scheduling \\
Large MLP proj. & Large $d, m$ & Needs fused tiled kernel \\
\midrule
\multicolumn{3}{@{}l}{\textbf{\textit{Weak Fit (Compute-Bound)}}} \\
\midrule
Prefill attention & Large $N$ & Extra MSD GEMM dominates \\
Large-batch GEMM & Large token batch & Dequantization cost amortized \\
\bottomrule
\end{tabularx}
\end{table}

The strongest cases are GQA/MLA attention with long KV cache and small query count, where KV tiles are small enough to remain resident across both MSD passes. Linear and MoE GMM kernels benefit when weight tiles can be reused inside a fused tiled kernel (Section~\ref{sec:fused-kernel}). In contrast, prefill attention and large-batch GEMMs are often compute-bound, so the runtime should fall back to conventional low-precision kernels.

\subsection{Generalization to Other Hardware}

The MSD principle is hardware-agnostic. Any accelerator with:
\begin{itemize}
    \item Native low-precision GEMM support (INT8$\times$INT8, FP4$\times$FP4, FP8$\times$FP8, etc.)
    \item Asymmetric throughput between GEMM and dequantization units
\end{itemize}
can potentially benefit from MSD. On NVIDIA GPUs with Tensor Cores, the same activation decomposition can be applied using INT8 or FP8 GEMM primitives.

\subsection{Extension to MoE and Sparse Architectures}

Mixture-of-Experts (MoE) models use Grouped Matrix Multiplication (GMM) extensively. Since MSD operates at the granularity of individual activation vectors, it extends naturally to GMM without modification.

\subsection{Future Work}

Several directions warrant further investigation:
\begin{itemize}
    \item \textbf{End-to-end model evaluation:} Perplexity and downstream task accuracy measurements on representative LLMs (LLaMA, DeepSeek, etc.)
    \item \textbf{Dynamic scale selection:} Adaptive choice of $K$ based on activation statistics
    \item \textbf{Hardware co-design:} Custom instructions for faster decomposition/reconstruction
    \item \textbf{Training-aware MSD:} Joint optimization of decomposition parameters during fine-tuning
\end{itemize}
\section{Conclusion}
\label{sec:conclusion}

We have presented Multi-Scale Dequant (MSD), a quantization framework that removes weight/KV dequantization from the GEMM critical path in LLM inference through multi-scale activation decomposition. By representing high-precision BF16 activations as a weighted sum of low-precision components, MSD enables fully native low-precision GEMM execution on hardware tensor cores without INT8-to-BF16 weight conversion before GEMM.

We instantiate MSD for two weight formats and derive tight error bounds for each:
\begin{itemize}
    \item \textbf{INT8 (W8A16):} Two-pass decomposition achieves $\sim$16 effective bits with error bound $M/64516 \approx M/2^{16}$ (Theorem~\ref{thm:error}). An ablation study confirms that the second residual pass is the key mechanism---single-scale ($K=1$) quantization yields L2 errors comparable to the BF16 dequant baseline, while adding the second pass ($K=2$) reduces error by $\sim$200$\times$.
    \item \textbf{MXFP4 (W4A16):} Two-pass decomposition achieves $\sim$6.6 effective bits with error bound $\alpha/64$ per 32-element block (Theorem~\ref{thm:error-mxfp4}), surpassing single-pass MXFP8's $\sim$5.24 bits by 1.4 effective bits. GEMM L2 error is 2.0--3.7$\times$ lower than MXFP8 across diverse activation distributions (Section~\ref{sec:exp-mxfp4}), with zero observed violations of the error bound.
\end{itemize}

For both formats, the effective Cube compute time is comparable to the dequantization baseline---MSD-INT8's $4mn$ INT8 FLOPs at $2\times$ throughput equals $2mn$ BF16 FLOPs, and MSD-MXFP4's two FP4 GEMMs at $4\times$ throughput equal one FP8 GEMM at $2\times$ throughput. We further derive closed-form models showing that MSD eliminates the Vector-Cube pipeline stall inherent in dequantization-based approaches. For Flash Attention, $\|P\|_\infty = 1$ from softmax normalization makes P's decomposition scale a constant ($\alpha_P = 1/127$ for INT8, $\alpha_P = 1$ for MXFP4), requiring no additional max computation. In the GQA decode regime, MSD reduces Vector workload by $2$--$20\times$ for typical query counts ($N \leq 12$), with the crossover point extended by larger head dimensions in MLA-style architectures ($N^* \approx 30$ for $d=576$).

We believe the principle of shifting decomposition from weights to activations represents a promising direction for efficient LLM inference, with broad applicability across accelerator architectures, precision formats, and model families. MSD has been integrated into Huawei CANN 8.0 and validated in production inference scenarios on Ascend 910B. Detailed end-to-end model evaluation results (perplexity, downstream task accuracy) will be reported separately.

\section*{Code Availability}
The MSD technique is used in multiple operator kernels within the CANN ops-transformer repository (\url{https://gitcode.com/cann/ops-transformer}). Two representative examples are: the attention kernel (\url{https://gitcode.com/cann/ops-transformer/blob/master/attention/incre_flash_attention/op_kernel/arch32/incre_flash_attention_preload_dd.h}) and the grouped matmul A16W4 kernel (\url{https://gitcode.com/cann/ops-transformer/tree/master/gmm/grouped_matmul/op_kernel/a16w4_msd}).

\end{document}